\documentclass{article}

    \usepackage[final,nonatbib]{neurips_2020}

\usepackage[numbers]{natbib}
\usepackage[utf8]{inputenc} \usepackage[T1]{fontenc}    \usepackage{hyperref}       \usepackage{url}            \usepackage{booktabs}       \usepackage{amsfonts}       \usepackage{nicefrac}       \usepackage{microtype}      \usepackage{commands}
\usepackage{graphicx}
\usepackage{multirow}
\usepackage{tikz}
\usetikzlibrary{shapes}
\usepackage{subcaption}
\usepackage{layouts}
\usepackage{adjustbox}
\usepackage{multirow}
\usepackage[linesnumbered, ruled, vlined, noline, noend]{algorithm2e}
\usepackage{soul}
\usepackage{transparent}
\usepackage{multibib}

\hypersetup{
    colorlinks=true,
}

\DeclareCaptionLabelFormat{andtable}{#1 #2 \& \tablename{} \thetable{}}
\DeclareCaptionLabelFormat{andfigure}{#1 #2 \& \figurename{} \thefigure{}}

\definecolor{myblue}{HTML}{1F77B4}
\definecolor{myred}{HTML}{D62728}
\definecolor{mypurple}{rgb}{0.52, 0.29222222, 0.39888889}
\newcommand{\hlmyblue}[1]{{\colorbox{myblue!20}{#1}}}
\newcommand{\hlmyred}[1]{{\colorbox{myred!20}{#1}}}
\newcommand{\hlmypurple}[1]{{\colorbox{mypurple!36}{#1}}}

\newcommand{\smtabbrev}{SMT}
\newcommand{\rsmtabbrev}{SST}
\makeatletter
\newcommand\thefontsize{The current font size is: \f@size pt}
\makeatother
\title{Gradient Estimation with Stochastic Softmax Tricks}

\author{
    Max B. Paulus\thanks{Equal Contribution. Correspondence to \texttt{max.paulus@inf.ethz.ch}, \texttt{choidami@cs.toronto.edu}.}\\
    ETH Z\"{u}rich\\
    \texttt{max.paulus@inf.ethz.ch}\\
    \And
    Dami Choi\footnotemark[1]\\
    University of Toronto\\
    \texttt{choidami@cs.toronto.edu}\\
    \AND
    Daniel Tarlow\\
    Google Research, Brain Team\\
    \texttt{dtarlow@google.com}\\
    \And
    Andreas Krause\\
    ETH Z\"{u}rich\\
    \texttt{krausea@ethz.ch}\\
    \And
    Chris J. Maddison\thanks{Work done partly at the Institute for Advanced Study, Princeton, NJ.}\\
    University of Toronto \& DeepMind\\
    \texttt{cmaddis@cs.toronto.edu}
}

\begin{document}

\maketitle

\begin{abstract}
The Gumbel-Max trick is the basis of many relaxed gradient estimators. These estimators are easy to implement and low variance, but the goal of scaling them comprehensively to large combinatorial distributions is still outstanding. Working within the perturbation model framework, we introduce stochastic softmax tricks, which generalize the Gumbel-Softmax trick to combinatorial spaces. Our framework is a unified perspective on existing relaxed estimators for perturbation models, and it contains many novel relaxations. We design structured relaxations for subset selection, spanning trees, arborescences, and others. When compared to less structured baselines, we find that stochastic softmax tricks can be used to train latent variable models that perform better and discover more latent structure.
\end{abstract}
 \section{Introduction}

Gradient computation is the methodological backbone of deep learning, but computing gradients is not always easy. Gradients with respect to parameters of the density of an integral are generally intractable, and one must resort to gradient estimators \citep{asmussen2007stochastic, mohamed2019gradientest}. Typical examples of objectives over densities are returns in reinforcement learning \citep{sutton2018reinforcement} or variational objectives for latent variable models \cite[e.g.,][]{kingma2014auto, rezende2014stochastic}. In this paper, we address {\em gradient estimation for discrete distributions} with an emphasis on latent variable models. We introduce a relaxed gradient estimation framework for combinatorial discrete distributions that generalizes the Gumbel-Softmax and related estimators \citep{maddison2016concrete, jang2016categorical}.

Relaxed gradient estimators incorporate bias in order to reduce variance. Most relaxed estimators are based on the Gumbel-Max trick \citep{luce1959individual, maddison2014astarsamp}, which reparameterizes distributions over one-hot binary vectors. The Gumbel-Softmax estimator is the simplest; it continuously approximates the Gumbel-Max trick to admit a reparameterization gradient \citep{kingma2014auto, rezende2014stochastic, ruiz2016generalized}. This is used to optimize the ``soft'' approximation of the loss as a surrogate for the ``hard'' discrete objective.

Adding structured latent variables to deep learning models is a promising direction for addressing a number of challenges:~improving interpretability (e.g., via latent variables for subset selection \citep{chen2018learning} or parse trees \cite{corro2018differentiable}), incorporating problem-specific constraints (e.g., via enforcing alignments \cite{mena2018learning}), and improving generalization (e.g., by modeling known algorithmic structure \cite{graves2014neural}).
Unfortunately, the vanilla Gumbel-Softmax cannot scale to distributions over large state spaces, and the development of structured relaxations has been piecemeal.

We introduce \emph{stochastic softmax tricks} (\rsmtabbrev s), which are a unified framework for designing structured relaxations of combinatorial distributions. They include relaxations for the above applications, as well as many novel ones.
To use an \rsmtabbrev{,} a modeler chooses from a class of models that we call \emph{stochastic argmax tricks} (\smtabbrev{}). These are instances of perturbation models \citep[e.g.,][]{papandreou2011perturb, hazan2012partition, tarlow2012randoms, gane2014learning}, and they induce a distribution over a finite set $\discreteset$ by optimizing
a linear objective (defined by random utility $U \in \R^n$) over $\discreteset$.
An \rsmtabbrev{} relaxes this \smtabbrev{} by combining a strongly convex regularizer with the random linear objective. The regularizer makes the solution a continuous, a.e. differentiable function of $U$ and appropriate for estimating gradients with respect to $U$'s parameters. The Gumbel-Softmax is a special case. Fig. \ref{fig:intro} provides a summary.

We test our relaxations in the Neural Relational Inference (NRI) \citep{kipf2018neural} and L2X \cite{chen2018learning} frameworks. Both NRI and L2X use variational losses over latent combinatorial distributions. When the latent structure in the model matches the true latent structure, we find that our relaxations encourage the unsupervised discovery of this combinatorial structure. This leads to models that are more interpretable and achieve stronger performance than less structured baselines. All proofs are in the Appendix.

\begin{figure}[t]
    \centering
    \begin{subfigure}[b]{0.246\textwidth}
        \includegraphics{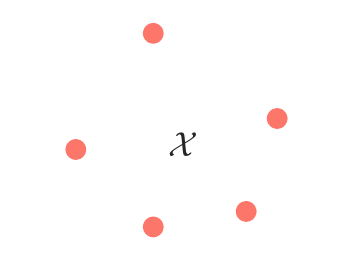}
        \caption*{Finite set}
    \end{subfigure}
    \hfill
    \begin{subfigure}[b]{0.245\textwidth}
        \includegraphics{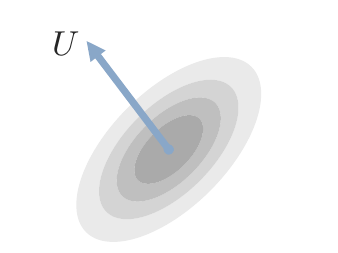}
        \caption*{Random utility}
    \end{subfigure}
    \hfill
    \begin{subfigure}[b]{0.245\textwidth}
        \includegraphics{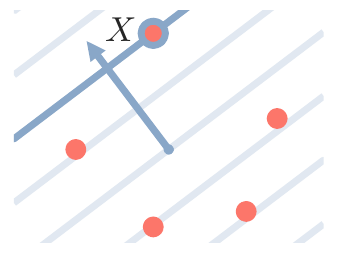}
        \caption*{Stoch. Argmax Trick}
    \end{subfigure}
    \hfill
    \begin{subfigure}[b]{0.245\textwidth}
        \includegraphics{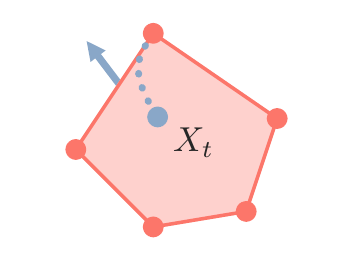}
        \caption*{Stoch. Softmax Trick}
    \end{subfigure}
    \caption{Stochastic softmax tricks relax discrete distributions that can be reparameterized as random linear programs. $X$ is the solution of a random linear program defined by a finite set $\discreteset$ and a random utility $U$ with parameters $\theta \in \R^m$. To design relaxed gradient estimators with respect to $\theta$, $X_{\temp}$ is the solution of a random convex program that continuously approximates $X$ from within the convex hull of $\discreteset$. The Gumbel-Softmax \citep{maddison2016concrete, jang2016categorical} is an example of a stochastic softmax trick.}\label{fig:intro}
    \vspace{-0.5\baselineskip}
\end{figure}
 \section{Problem Statement}
\label{sec:problemstmt}
Let $\abstractset$ be a non-empty, finite set of combinatorial objects, e.g. the spanning trees of a graph. To represent $\abstractset$, define the embeddings $\discreteset \subseteq \R^n$ of $\abstractset$ to be the image $\{ \embed(y) \mid y \in \abstractset\}$ of some embedding function $\embed : \abstractset \to \R^n$.\footnote{This is equivalent to the notion of sufficient statistics \cite{wainwright2008graphical}. We draw a distinction only to avoid confusion, because the distributions $p_{\theta}$ that we ultimately consider are not necessarily from the exponential family.} For example, if $\abstractset$ is the set of spanning trees of a graph with edges $E$, then we could enumerate $y_1, \ldots, y_{|\abstractset|}$ in $\abstractset$ and let $\embed(y)$ be the one-hot binary vector of length $|\abstractset|$, with $\embed(y)_i = 1$ iff $y = y_i$. This requires a very large ambient dimension $n = |\abstractset|$. Alternatively, in this case we could use a more efficient, structured representation: $\embed(y)$ could be a binary indicator vector of length $|E| \ll |\abstractset|$, with $\embed(y)_e = 1$ iff edge $e$ is in the tree $y$. See Fig. \ref{fig:embeddings} for visualizations and additional examples of structured binary representations. We assume that $\discreteset$ is convex independent.\footnote{Convex independence is the analog of linear independence for convex combinations.} 

Given a probability mass function $p_{\theta} : \discreteset \to (0, 1]$ that is differentiable in $\theta \in \R^m$, a loss function $\loss : \R^n \to \R$, and $X \sim p_{\theta}$, our ultimate goal is gradient-based optimization of $\expect[\loss(X)]$. Thus, we are concerned in this paper with the problem of estimating the derivatives of the expected loss,
\begin{equation}
    \label{eq:problem}
    \frac{d}{d \theta}\expect[\loss(X)] = \frac{d}{d \theta} \left(\sum\nolimits_{x
    \in \discreteset} \loss(x) p_{\theta}(x)\right).
\end{equation}

\section{Background on Gradient Estimation}
\label{sec:background}
Relaxed gradient estimators assume that $\loss$ is differentiable and use a change of variables to remove the dependence of $p_{\theta}$ on $\theta$, known as the reparameterization trick \citep{kingma2014auto, rezende2014stochastic}. The Gumbel-Softmax trick (GST) \citep{maddison2016concrete, jang2016categorical} is a simple relaxed gradient estimator for one-hot embeddings, which is based on the Gumbel-Max trick (GMT) \citep{luce1959individual, maddison2014astarsamp}. Let $\discreteset$ be the one-hot embeddings of $\abstractset$ and $p_{\theta}(x) \propto \exp(x^T\theta)$. The GMT is the following identity: for $X\sim p_{\theta}$ and $G_i + \theta_i \sim \Gumbel(\theta_i)$ indep.,
\begin{align}
\label{eq:gumbelmaxtrick} X \overset{d}{=} \arg \max\nolimits_{x \in \discreteset} \, (G+\theta)^T x.
\end{align}
Ideally, one would have a reparameterization estimator, $\expect[d \loss(X)/d \theta] = d \expect[\loss(X)]/d \theta$,\footnote{For a function $f(x_1, x_2)$, $\partial f(z_1, z_2) / \partial x_1$ is the partial derivative (e.g., a gradient vector) of $f$ in the first variable evaluated at $z_1, z_2$. $d f(z_1, z_2) / d x_1$ is the total derivative of $f$ in $x_1$ evaluated at $z_1, z_2$. For example, if $x = f(\theta)$, then $ d g(x, \theta)/ d\theta = (\partial g(x, \theta)/\partial x) (d f(\theta)/ d\theta) + \partial g(x, \theta)/\partial\theta$.} using the right-hand expression in \eqref{eq:gumbelmaxtrick}. Unfortunately, this fails. The problem is not the lack of differentiability, as normally reported. In fact, the argmax is differentiable almost everywhere. Instead it is the jump discontinuities in the argmax that invalidate this particular exchange of expectation and differentiation \citep[][Chap. 7.2]{lee2018reparameterization, asmussen2007stochastic}. The GST estimator \citep{maddison2016concrete, jang2016categorical} overcomes this by using the tempered softmax, $\softmax_{\temp}(u)_i = \exp(u_i/\temp) / \sum_{j=1}^n \exp(u_j/\temp)$ for $u \in \R^n, \temp > 0$, to continuously approximate $X$,
\begin{align}
\label{eq:gumbelsoftmaxestimator}
X_{\temp} = \softmax_{\temp}(G + \theta).
\end{align}
The relaxed estimator is $d \loss(X_t) / d \theta$. While this is a biased estimator of \eqref{eq:problem}, it is an unbiased estimator of $d\mathbb{E}[\loss(X_{\temp})]/d\theta$ and $X_t \to X$ a.s. as $t \to 0$. Thus, $d \loss(X_t) / d \theta$ is used for optimizing $\expect[\loss(X_{\temp})]$ as a surrogate for $\expect[\loss(X)]$, on which the final model is evaluated.

The score function estimator \citep{glynn1990likelihood, williams1992simple}, $\loss(X) \, \partial \log p_{\theta}(X) / \partial \theta$, is the classical alternative. It is a simple, unbiased estimator, but without highly engineered control variates, it suffers from high variance \citep{mnih2014neural}. Building on the score function estimator are a variety of estimators that require multiple evaluations of $\loss$ to reduce variance \citep{DBLP:journals/corr/GuLSM15, tucker2017rebar, grathwohl2018backpropagation, yin2018arm, Kool2020Estimating, aueb2015local}.  The advantages of relaxed estimators are the following: they only require a single evaluation of $\loss$, they are easy to implement using modern software packages \citep{abadi2016tensorflow, paszke2017automatic, jax2018github}, and, as reparameterization gradients, they tend to have low variance \citep{gal2016uncertainty}.

\begin{figure}[t]
\centering
\tabskip=0pt
\valign{#\cr
\hbox{
        \begin{subfigure}[b]{0.225\textwidth}
           \centering
            \includegraphics[width=0.75in]{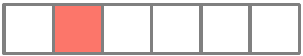}
            \caption*{One-hot vector}
        \end{subfigure}
    }
\vfill
\hbox{
        \begin{subfigure}[b]{0.225\textwidth}
            \centering
            \includegraphics[width=0.75in]{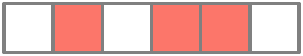}
            \caption*{$k$-hot vector}
        \end{subfigure}
    }\cr
\hbox{
    \begin{subfigure}[b]{0.225\textwidth}
        \centering
        \includegraphics[width=0.75in]{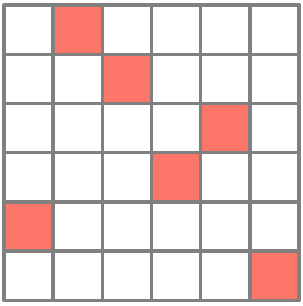}
        \caption*{Permutation matrix}
    \end{subfigure}
}\cr
\hbox{
    \begin{subfigure}[b]{0.225\textwidth}
        \centering
        \includegraphics[width=0.75in]{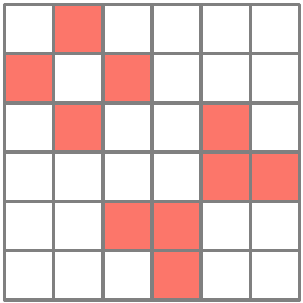}
        \caption*{Spanning tree adj. matrix}
    \end{subfigure}
}\cr
\hbox{
    \begin{subfigure}[b]{0.225\textwidth}
        \centering
        \includegraphics[width=0.75in]{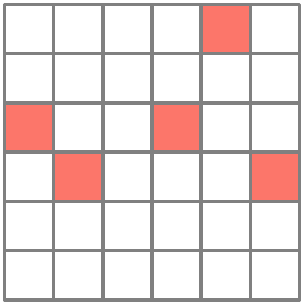}
        \caption*{Arborescence adj. matrix}
    \end{subfigure}
}\cr
}
\caption{Structured discrete objects can be represented by binary arrays. In these graphical representations, color indicates 1 and no color indicates 0. For example, ``Spanning tree'' is the adjacency matrix of an undirected spanning tree over 6 nodes; ``Arborescence'' is the adjacency matrix of a directed spanning tree rooted at node 3.}
\label{fig:embeddings}
\end{figure}
 \section{Stochastic Argmax Tricks}
\label{sec:smts}

Simulating a GST requires enumerating $|\abstractset|$ random variables, so it cannot scale. We overcome this by identifying generalizations of the GMT that can be relaxed and that scale to large $\abstractset$s by exploiting structured embeddings $\discreteset$. We call these \emph{stochastic argmax tricks} (\smtabbrev{s}), because they are perturbation models \citep{tarlow2012randoms, gane2014learning}, which can be relaxed into stochastic softmax tricks (Section \ref{sec:rsmts}).
\begin{definition}
\label{def:smt}
Given a non-empty, convex independent, finite set $\discreteset \subseteq \R^n$ and a random utility $U$ whose distribution is parameterized by $\theta \in \R^m$, a {\em stochastic argmax trick} for $X$ is the linear program,
\begin{equation}
    \label{eq:smt}
    X = \arg \max\nolimits_{x \in \discreteset} \,  U^T x.
\end{equation}
\end{definition}
The GMT is recovered with one-hot $\discreteset$ and $U \sim \Gumbel(\theta)$. We assume that \eqref{eq:smt} is a.s.~unique, which is guaranteed if $U$ a.s.~never lands in any particular lower dimensional subspace (Prop. \ref{prop:noisedistribution}, App. \ref{supp:sec:proofs}). Because efficient linear solvers are known for many structured $\discreteset$, \smtabbrev{s} are capable of scaling to very large $\abstractset$ \citep{schrijver2003combinatorial, kolmogorov2006convergent, koller2009probabilistic}. For example, if $\discreteset$ are the edge indicator vectors of spanning trees $\abstractset$, then \eqref{eq:smt} is the maximum spanning tree problem, which is solved by Kruskal's algorithm \citep{kruskal1956shortest}.

The role of the \smtabbrev{} in our framework is to reparameterize $p_{\theta}$ in \eqref{eq:problem}. Ideally, \emph{given} $p_{\theta}$, there would be an efficient (e.g., $\mathcal{O}(n)$) method for simulating \emph{some} $U$ such that the marginal of $X$ in \eqref{eq:smt} is $p_{\theta}$. The GMT shows that this is possible for one-hot $\discreteset$, but the situation is not so simple for structured $\discreteset$. Characterizing the marginal of $X$ in general is difficult \cite{tarlow2012randoms, hazan2013perturb}, but $U$ that are efficient to sample from typically induce conditional independencies in $p_{\theta}$ \citep{gane2014learning}. Therefore, we are not able to reparameterize an arbitrary $p_{\theta}$ on structured $\discreteset$. Instead, for structured $\discreteset$ we \emph{assume} that $p_{\theta}$ is reparameterized by \eqref{eq:smt}, and treat $U$ as a modeling choice. Thus, we caution against the standard approach of taking $U \sim \Gumbel(\theta)$ or $U \sim \Normal(\theta, \sigma^2I)$ without further analysis. Practically, in experiments we show that the difference in noise distribution can have a large impact on quantitative results. Theoretically, we show in App. \ref{supp:sec:fieldguide} that an \smtabbrev{} over directed spanning trees with negative exponential utilities has a more interpretable structure than the same \smtabbrev{} with Gumbel utilities. \section{Stochastic Softmax Tricks}
\label{sec:rsmts}

If we assume that $X \sim p_{\theta}$ is reparameterized as an \smtabbrev{}, then a stochastic softmax trick (\rsmtabbrev{)} is a random convex program with a solution that relaxes $X$. An \rsmtabbrev{} has a valid reparameterization gradient estimator. Thus, we propose using \rsmtabbrev{s} as surrogates for estimating gradients of \eqref{eq:problem}, a generalization of the Gumbel-Softmax approach. Because we want gradients with respect to $\theta$, we assume that $U$ is also reparameterizable.

Given an \smtabbrev{}, an \rsmtabbrev{} incorporates a strongly convex regularizer to the linear objective, and expands the state space to the convex hull of the embeddings $\discreteset = \{x_1, \ldots, x_m\} \subseteq \R^n$,
\begin{equation}
  P := \hull(\discreteset) := \left\{\sum\nolimits_{i=1}^m \lambda_i x_i  \, \middle\vert  \, \lambda_i \geq 0, \,  \sum\nolimits_{i=1}^m \lambda_i = 1\right\}.
\end{equation}
Expanding the state space to a convex polytope makes it path-connected, and the strongly convex regularizer ensures that the solutions are continuous over the polytope.
\begin{definition}
\label{def:rsmt}
Given a stochastic argmax trick $(\discreteset, U)$ where $P := \hull(\discreteset)$ and a proper, closed, strongly convex function $f : \R^n \to \{\R, \infty\}$ whose domain contains the relative interior of $P$, a {\em stochastic softmax trick} for $X$ at temperature $\temp > 0$ is the convex program,
\begin{equation}
    \label{eq:rsmt}
    X_{\temp} = \arg \max_{x \in P} \,  U^T x - \temp f(x)
\end{equation}
\end{definition}
For one-hot $\discreteset$, the Gumbel-Softmax is a special case of an \rsmtabbrev{} where $P$ is the probability simplex, $U \sim \Gumbel(\theta)$, and $f(x) = \sum_{i} x_i \log(x_i)$. Objectives like \eqref{eq:rsmt} have a long history in convex analysis \citep[e.g.,][Chap. 12]{rockafellar1970convex} and machine learning \citep[e.g.,][Chap. 3]{wainwright2008graphical}. In general, the difficulty of computing the \rsmtabbrev{} will depend on the interaction between $f$ and $\discreteset$.

$X_{\temp}$ is suitable as an approximation of $X$. At positive temperatures $\temp$, $X_{\temp}$ is a function of $U$ that ranges over the faces and relative interior of $P$. The degree of approximation is controlled by the temperature parameter, and as $\temp \to 0^+$, $X_{\temp}$ is driven to $X$ a.s.
\begin{restatable}{proposition}{approximation}
\label{prop:approximation}
If $X$ in Def. \ref{def:smt} is a.s. unique, then  for $X_t$ in Def. \ref{def:rsmt}, $\lim_{t \to 0^+} X_{\temp} = X$ a.s. If additionally $\loss : P \to \R$ is bounded and continuous, then $\lim_{t \to 0^+} \expect[\loss(X_{\temp})] = \expect[\loss(X)]$.
\end{restatable}
It is common to consider temperature parameters that interpolate between marginal inference and a deterministic, most probable state. While superficially similar, our relaxation framework is different; as $\temp \to 0^+$, an \rsmtabbrev{} approaches \emph{a sample from the \smtabbrev{} model} as opposed to a deterministic state.

$X_{\temp}$ also admits a reparameterization trick. The \rsmtabbrev{} reparameterization gradient estimator given by,
\begin{equation}
  \label{eq:solution}
  \frac{d \loss(X_{\temp})}{d \theta} = \frac{\partial \loss(X_{\temp})}{\partial X_{\temp}} \frac{\partial X_{\temp}}{\partial U} \frac{d U}{d \theta}.
\end{equation}
If $\loss$ is differentiable on $P$, then this is an unbiased estimator\footnote{Technically, one needs an additional local Lipschitz condition for $\loss(X_{\temp})$ in $\theta$ \citep[Prop. 2.3, Chap. 7]{asmussen2007stochastic}.} of the gradient $d\mathbb{E}[\loss(X_{\temp})] / d \theta$, because $X_{\temp}$ is continuous and a.e. differentiable:
\begin{restatable}{proposition}{relaxation}
\label{prop:relaxation}
$X_{\temp}$ in Def. \ref{def:rsmt} exists, is unique, and is a.e. differentiable and continuous in $U$.
\end{restatable}
In general, the Jacobian $\partial X_{\temp} / \partial U$  will need to be derived separately given a choice of $f$ and $\discreteset$. However, as pointed out by \citep{domke2010impdiff}, because the Jacobian of $X_{\temp}$ symmetric \citep[][Cor. 2.9]{rockafellar1999second}, local finite difference approximations can be used to approximate $d \loss(X_{\temp})/ d U$ (App. \ref{supp:sec:exp_details}). These finite difference approximations only require two additional calls to a solver for \eqref{eq:rsmt} and do not require additional evaluations of $\loss$. We found them to be helpful in a few experiments (c.f., Section \ref{sec:experiments}).

There are many, well-studied $f$ for which \eqref{eq:rsmt} is efficiently solvable. If $f(x) = \lVert x \rVert^2 / 2$, then $X_{\temp}$ is the Euclidean projection of $U/t$ onto $P$. Efficient projection algorithms exist for some convex sets \citep[see][and references therein]{wolfe1976finding, duchi2008efficient, liu2009efficient, blondel2019structured}, and more generic algorithms exist that only call linear solvers as subroutines \citep{niculae2018sparsemap}. In some of the settings we consider, generic negative-entropy-based relaxations are also applicable. We refer to relaxations with $f(x) = \sum\nolimits_{i=1}^n x_i \log(x_i)$ as \emph{categorical entropy relaxations} \citep[e.g.,][]{blondel2019structured, blondel2020learning}.
We refer to relaxations with $f(x) = \sum\nolimits_{i=1}^n x_i \log (x_i) + (1-x_i) \log(1-x_i)$ as \emph{binary entropy relaxations} \cite[e.g.,][]{amos2019limited}.

Marginal inference in exponential families is a rich source of \rsmtabbrev{} relaxations. Consider an exponential family over the finite set $\discreteset$ with natural parameters $u/\temp \in \R^n$ such that the probability of $x \in \discreteset$ is proportional to $\exp(u^Tx/\temp)$. The \emph{marginals} $\mu_{\temp} : \R^n \to \hull(\discreteset)$ of this family are solutions of a convex program in exactly the form \eqref{eq:rsmt} \citep{wainwright2008graphical}, i.e., there exists $A^* : \hull(\discreteset) \to \{\R, \infty\}$ such that,
\begin{equation}
  \label{eq:expfamilymarg}
  \mu_{\temp}(u) := \sum\nolimits_{x \in \discreteset} \frac{x\exp(u^Tx/\temp)}{\sum_{y \in \discreteset} \exp(u^T y/\temp)} = \arg\max_{x \in P} u^Tx - \temp A^*(x).
\end{equation}
 The definition of $A^*$, which generates $\mu_{\temp}$ in \eqref{eq:expfamilymarg}, can be found in \citep[][Thm. 3.4]{wainwright2008graphical}. $A^*$ is a kind of negative entropy and in our case it satisfies the assumptions in Def. \ref{def:rsmt}. Computing $\mu_{\temp}$ amounts to marginal inference in the exponential family, and efficient algorithms are known in many cases \citep[see][]{wainwright2008graphical, koller2009probabilistic}, including those we consider. We call $X_{\temp} = \mu_{\temp}(U)$ the \emph{exponential family entropy relaxation}.

Taken together, Prop. \ref{prop:approximation} and \ref{prop:relaxation} suggest our proposed use for \rsmtabbrev s: optimize $\expect[\loss(X_{\temp})]$ at a positive temperature, where unbiased gradient estimation is available, but evaluate $\expect[\loss(X)]$. We find that this works well in practice if the temperature used during optimization is treated as a hyperparameter and selected over a validation set. It is worth emphasizing that the choice of relaxation is unrelated to the distribution $p_{\theta}$ of $X$ in the corresponding \smtabbrev{}. $f$ is not only a modeling choice; it is a computational choice that will affect the cost of computing \eqref{eq:rsmt} and the quality of the gradient estimator.
 \section{Examples of Stochastic Softmax Tricks}
\label{sec:examples}

\begin{figure}[t]
    \centering
    \includegraphics[scale=0.66666]{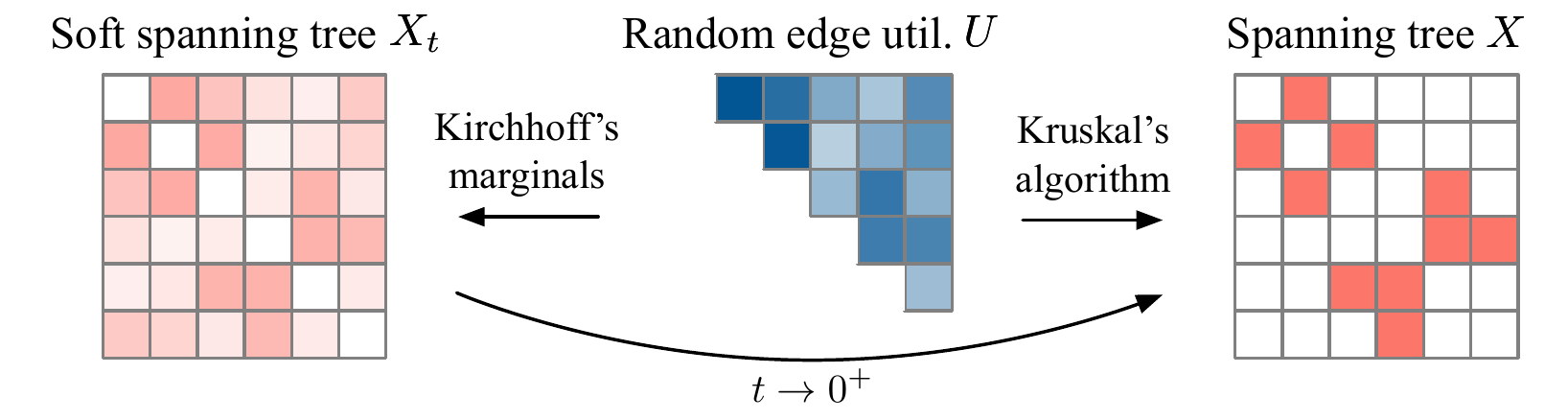}
    \caption{An example realization of a spanning tree \rsmtabbrev{} for an undirected graph. Middle: Random undirected edge utilities. Left: The random soft spanning tree $X_{\temp}$, represented as a weighted adjacency matrix, can be computed via Kirchhoff's Matrix-Tree theorem. Right: The random spanning tree $X$, represented as an adjacency matrix, can be computed with Kruskal's algorithm.}
    \label{fig:spanningtree}
    \vspace{-0.5\baselineskip}
\end{figure}

The Gumbel-Softmax \citep{maddison2016concrete, jang2016categorical} introduced neither the Gumbel-Max trick nor the softmax. The novelty of this work is neither the pertubation model framework nor the relaxation framework in isolation, but their combined use for gradient estimation. Here we layout some example \rsmtabbrev{s}, organized by the set $\abstractset$ with a choice of embeddings $\discreteset$. Bold italics indicates previously described relaxations, most of which are bespoke and not describable in our framework. Italics indicates our novel \rsmtabbrev s used in our experiments; some of these are also novel perturbation models. A complete discussion is in App. \ref{supp:sec:fieldguide}.

\textbf{Subset selection.} $\discreteset$ is the set of binary vectors indicating membership in the subsets of a finite set $S$. \emph{Indep. $S$} uses $U \sim \Logistic(\theta)$ and a binary entropy relaxation. $X$ and $X_{\temp}$ are computed with a dimension-wise step function or sigmoid, resp.

\textbf{$\mathbf{k}$-Subset selection.} $\discreteset$ is the set of binary vectors with a $k$-hot binary vectors indicating membership in a $k$-subset of a finite set $S$. All of the following \smtabbrev s use $U \sim \Gumbel(\theta)$. Our \rsmtabbrev{s} use the following relaxations: euclidean \citep{amos2017optnet} and categorical \citep{martins2017learning}, binary \citep{amos2019limited}, and exponential family \citep{swersky2012cardinality} entropies. $X$ is computed by sorting $U$ and setting the top $k$ elements to 1 \citep{blondel2019structured}. \emph{$R$ Top $k$} refers to our \rsmtabbrev{} with relaxation $R$. \emph{\textbf{L2X}} \citep{chen2018learning} and \emph{\textbf{SoftSub}} \citep{xie2019reparameterizable} are bespoke relaxations.

\textbf{Correlated $\mathbf{k}$-subset selection.} $\discreteset$ is the set of $(2n-1)$-dimensional binary vectors with a $k$-hot cardinality constraint on the first $n$ dimensions and a constraint that the $n-1$ dimensions indicate correlations between adjacent dimensions in the first $n$, i.e. the vertices of the correlation polytope of a chain \citep[][Ex. 3.8]{wainwright2008graphical} with an added cardinality constraint \citep{mezuman2013tighter}. \emph{Corr. Top $k$} uses $U_{1:n} \sim \Gumbel(\theta_{1:n})$, $U_{n+1:2n-1} = \theta_{n+1:2n-1}$, and the exponential family entropy relaxation. $X$ and $X_{\temp}$ can be computed with dynamic programs \citep{tarlow2012fast}, see App. \ref{supp:sec:fieldguide}.

\textbf{Perfect Bipartite Matchings.} $\discreteset$ is the set of $n \times n$ permutation matrices representing the perfect matchings of the complete bipartite graph $K_{n,n}$.  The \emph{\textbf{Gumbel-Sinkhorn}} \citep{mena2018learning} uses $U \sim \Gumbel(\theta)$ and a Shannon entropy relaxation. $X$ can be computed with the Hungarian method \citep{kuhn1955hungarian} and $X_{\temp}$ with the Sinkhorn algorithm \citep{sinkhorn1967concerning}. \emph{\textbf{Stochastic NeuralSort}} \citep{grover2018stochastic} uses correlated Gumbel-based utilities that induce a Plackett-Luce model and a bespoke relaxation.

\textbf{Undirected spanning trees.} Given a graph $(V, E)$, $\discreteset$ is the set of binary indicator vectors of the edge sets $T \subseteq E$ of undirected spanning trees. \emph{Spanning Tree} uses $U \sim \Gumbel(\theta)$ and the exponential family entropy relaxation. $X$ can be computed with Kruskal's algorithm \citep{kruskal1956shortest}, $X_{\temp}$ with Kirchhoff's matrix-tree theorem \citep[][Sec. 3.3]{koo2007matrixtree}, and both are represented as adjacency matrices, Fig. \ref{fig:spanningtree}.

\textbf{Rooted directed spanning trees.} Given a graph $(V, E)$, $\discreteset$ is the set of binary indicator vectors of the edge sets $T \subseteq E$ of $r$-rooted, directed spanning trees. \emph{Arborescence} uses  $U \sim \Gumbel(\theta)$ or $-U \sim \exponential(\theta)$ or $U\sim \Normal(\theta, I)$ and an exponential family entropy relaxation. $X$ can be computed with the Chu-Liu-Edmonds algorithm \citep{chu1965shortest, edmonds1967optimum}, $X_{\temp}$ with a directed version of Kirchhoff's matrix-tree theorem \citep[][Sec. 3.3]{koo2007matrixtree}, and both are represented as adjacency matrices. \emph{\textbf{Perturb \& Parse}}  \citep{corro2018differentiable} further restricts $\discreteset$ to be projective trees, uses $U \sim \Gumbel(\theta)$, and uses a bespoke relaxation. \section{Related Work}

Here we review perturbation models (PMs) and methods for relaxation more generally. \smtabbrev{s} are a subclass of PMs, which draw samples by optimizing a random objective. Perhaps the earliest example comes from Thurstonian ranking models \cite{thurstone1927law}, where a distribution over rankings is formed by sorting a vector of noisy scores. Perturb \& MAP models \cite{papandreou2011perturb,hazan2012partition} were designed to approximate the Gibbs distribution over a combinatorial output space using low-order, additive Gumbel noise.
Randomized Optimum models \cite{tarlow2012randoms,gane2014learning} are the most general class, which include non-additive noise distributions and non-linear objectives.
Recent work \citep{lorberbom2019direct} uses PMs to construct finite difference approximations of the expected loss' gradient. It requires optimizing a non-linear objective over $\discreteset$, and making this applicable to our settings would require significant innovation.

Using \rsmtabbrev{}s for gradient estimation requires differentiating through a convex program. This idea is not ours and is enjoying renewed interest in \cite{cvxpylayers2019, agrawal2019differentiating, amos2019differentiable}. In addition, specialized solutions have been proposed for quadratic programs \cite{amos2017optnet, martins2016softmax, blondel2020fast} and linear programs with entropic regularizers over various domains \cite{martins2017learning, amos2019limited, adams2011ranking, mena2018learning, blondel2020fast}.
In graphical modeling, several works have explored differentiating through marginal inference \cite{domke2010impdiff,ross-cvpr-11,poon2011sum,domke2013learning,swersky2012cardinality,djolonga2017differentiable} and our exponential family entropy relaxation builds on this work. The most superficially similar work is \citep{2020arXiv200208676B}, which uses noisy utilities to smooth the solutions of linear programs. In \citep{2020arXiv200208676B}, the noise is a tool for approximately relaxing a deterministic linear program. Our framework uses relaxations to approximate \emph{stochastic} linear programs.
 \section{Experiments}
\label{sec:experiments}
Our goal in these experiments was to evaluate the use of \rsmtabbrev{s} for learning distributions over structured latent spaces in deep structured models. We chose frameworks (NRI \citep{kipf2018neural}, L2X \citep{chen2018learning}, and a latent parse tree task) in which relaxed gradient estimators are the methods of choice, and investigated the effects of $\discreteset$, $f$, and $U$ on the task objective and on the unsupervised structure discovery. For NRI, we also implemented the standard single-loss-evaluation score function estimators (REINFORCE \citep{williams1992simple} and NVIL  \citep{mnih2014neural}), and the best \rsmtabbrev{} outperformed these baselines both in terms of average performance and variance, see App. \ref{supp:sec:addresults}. All \rsmtabbrev{} models were trained with the ``soft'' \rsmtabbrev{} and evaluated with the ``hard'' \smtabbrev{}. We optimized hyperparameters (including fixed training temperature $\temp$) using random search over multiple independent runs. We selected models on a validation set according to the best objective value obtained during training. All reported values are measured on a test set. Error bars are bootstrap standard errors over the model selection process. We refer to \rsmtabbrev{s} defined in Section
~\ref{sec:examples} with italics.  Details are in App. \ref{supp:sec:exp_details}. Code is available at \url{https://github.com/choidami/sst}.

\subsection{Neural Relational Inference (NRI) for Graph Layout}
\begin{table}[t]
  \refstepcounter{figure}
  \label{fig:graph_layout}
  \captionsetup{labelformat=andfigure}
  \caption{\emph{Spanning Tree} performs best on structure recovery, despite being trained on the ELBO. Test ELBO and structure recovery metrics are shown from models selected on valid. ELBO. Below: Test set example where \emph{Spanning Tree} recovers the ground truth latent graph perfectly.
  \vspace{5pt}}
  \label{table:graph_layout}
  \begin{subfigure}{\textwidth}
    \centering
    \begin{small}
    \adjustbox{max width=\textwidth}{    \begin{tabular}{@{}lcccccc@{}}
    \toprule
     & \multicolumn{3}{c}{$T=10$} & \multicolumn{3}{c}{$T=20$} \\ 
     \cmidrule(l){2-4} 
     \cmidrule(l){5-7}     
     Edge Distribution & ELBO & Edge Prec. & Edge Rec. & ELBO & Edge Prec. & Edge Rec. \\ 
     \midrule
    \emph{Indep. Directed Edges} \citep{kipf2018neural} 
    & $-1370 \pm 20$ 
    & $48 \pm 2$ 
    & $\mathbf{93 \pm 1}$
    & $-1340 \pm160$ 
    & $97 \pm 3$ 
    & $\mathbf{99 \pm 1}$  \\
    \emph{E.F. Ent. Top $|V|-1$}
    & $-2100 \pm 20$ 
    & $41 \pm 1$ 
    & $41 \pm 1$ 
    & $-1700 \pm 320$ 
    & $98\pm6$
    & $98\pm6$ \\
    \emph{Spanning Tree} 
    & $\mathbf{-1080\pm110}$ 
    & $\mathbf{91\pm3}$ 
    & $91\pm3$ 
    & $\mathbf{-1280\pm10}$ 
    & $\mathbf{99\pm1}$ 
    & $\mathbf{99\pm1}$ \\ 
    \bottomrule
    \end{tabular}}    \end{small}
  \end{subfigure}

  \begin{subfigure}{\textwidth}
      \centering
      \begin{subfigure}[b]{0.246\textwidth}
          \includegraphics{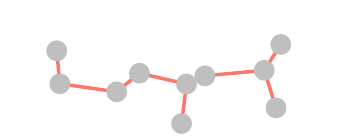}
          \caption*{Ground Truth}
      \end{subfigure}
      \hfill
      \begin{subfigure}[b]{0.245\textwidth}
          \includegraphics{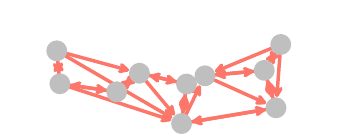}
          \caption*{\emph{Indep. Directed Edges}}
      \end{subfigure}
      \hfill
      \begin{subfigure}[b]{0.245\textwidth}
          \includegraphics{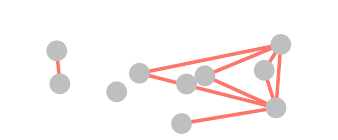}
          \caption*{\emph{E.F. Ent. Top $|V|-1$}}
      \end{subfigure}
      \hfill
      \begin{subfigure}[b]{0.245\textwidth}
          \includegraphics{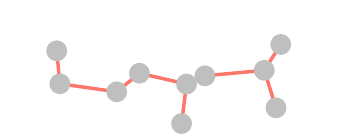}
          \caption*{\emph{Spanning Tree}}
      \end{subfigure}
  \end{subfigure}
  \vspace{-\baselineskip}
\end{table}
 
With NRI we investigated the use of \rsmtabbrev{}s for latent structure recovery and final performance. NRI is a graph neural network (GNN) model that samples a latent interaction graph $G = (V, E)$ and runs messages over the adjacency matrix to produce a distribution over an interacting particle system. NRI is trained as a variational autoencoder to maximize a lower bound (ELBO) on the marginal log-likelihood of the time series. We experimented with three \rsmtabbrev{s} for the encoder distribution: \emph{Indep. Binary} over directed edges, which is the baseline NRI encoder \citep{kipf2018neural}, \emph{E.F. Ent. Top $|V|-1$} over undirected edges, and \emph{Spanning Tree} over undirected edges. We computed the KL with respect to the random utility $U$ for all \rsmtabbrev{s}; see App. \ref{supp:sec:exp_details} for details. Our dataset consisted of latent prior spanning trees over 10 vertices sampled from the $\Gumbel(0)$ prior. Given a tree, we embed the vertices in $\R^2$ by applying $T \in \{10, 20\}$ iterations of a force-directed algorithm \citep{fruchterman1991graph}. The model saw particle locations at each iteration, not the underlying spanning tree.

We found that \emph{Spanning Tree} performed best, improving on both ELBO and the recovery of latent structure over the baseline \citep{kipf2018neural}. For structure recovery, we measured edge precision and recall against the ground truth adjacency matrix. It recovered the edge structure well even when given only a short series ($T=10$, Fig. \ref{fig:graph_layout}). Less structured baselines were only competitive on longer time series.
 \subsection{Unsupervised Parsing on ListOps}
We investigated the effect of $\discreteset{}$'s structure and of the utility distribution in a latent parse tree task. We used a simplified variant of the ListOps dataset \cite{nangia2018listops}, which contains sequences of prefix arithmetic expressions, e.g., \texttt{max[ 3 min[ 8 2 ]]}, that evaluate to an integer in $[0, 9]$. The arithmetic syntax induces a directed spanning tree rooted at its first token with directed edges from operators to operands. We modified the data by removing the \texttt{summod} operator, capping the maximum depth of the ground truth dependency parse, and capping the maximum length of a sequence. This simplifies the task considerably, but it makes the problem accessible to GNN models of fixed depth. Our models used a bi-LSTM encoder to produce a distribution over edges (directed or undirected) between all pairs of tokens, which induced a latent (di)graph. Predictions were made from the final embedding of the first token after passing messages in a GNN architecture over the latent graph. For undirected graphs, messages were passed in both directions. We experimented with the following \rsmtabbrev{s} for the edge distribution: \emph{Indep. Undirected Edges}, \emph{Spanning Tree}, \emph{Indep. Directed Edges}, and \emph{Arborescence} (with three separate utility distributions). \emph{Arborescence} was rooted at the first token. For baselines we used an unstructured LSTM and the GNN over the ground truth parse. All models were trained with cross-entropy to predict the integer evaluation of the sequence.

The best performing models were structured models whose structure better matched the true latent structure (Table \ref{table:listops_perform}). For each model, we measured the accuracy of its prediction (task accuracy). We measured both precision and recall with respect to the ground truth parse's adjacency matrix. \footnote{We exclude edges to and from the closing symbol ``$]$''. Its edge assignments cannot be learnt from the task objective, because the correct evaluation of an operation does not depend on the closing symbol.} Both tree-structured \rsmtabbrev{s} outperformed their independent edge counterparts on all metrics. Overall, \emph{Arborescence} achieved the best performance in terms of task accuracy and structure recovery. We found that the utility distribution significantly affected performance (Table \ref{table:listops_perform}). For example, while negative exponential utilities induce an interpretable distribution over arborescences, App. \ref{supp:sec:fieldguide}, we found that the multiplicative parameterization of exponentials made it difficult to train competitive models. Despite the LSTM baseline performing well on task accuracy, \emph{Arborescence} additionally learns to recover much of the latent parse tree.

\begin{table}
\caption{Matching ground truth structure (non-tree $\to$ tree) improves performance on ListOps. The utility distribution impacts performance. Test task accuracy and structure recovery metrics are shown from models selected on valid. task accuracy. Note that because we exclude edges to and from the closing symbol ``$]$'', recall is not equal to twice of precision for \emph{Spanning Tree} and precision is not equal to recall for \emph{Arborescence}.}
\label{table:listops_perform}
\begin{center}
\begin{small}
\begin{tabular}{@{}llccc@{}}
\toprule
Model & Edge Distribution & Task Acc. & Edge Precision & Edge Recall \\ 
\midrule
LSTM 
	& --- 
	& $92.1 \pm0.2$ 
	& --- 
	& ---  
\\ 
\cmidrule[0.15pt]{1-5}
\multirow{2}{*}{\shortstack[l]{GNN on\\latent graph}} 
& \emph{Indep. Undirected Edges} 
	& $89.4\pm0.6$ 
	& $20.1\pm2.1$ 
	& $45.4\pm6.5$ \\ 
& \emph{Spanning Tree} 
	& $91.2\pm 1.8$ 
	& $33.1\pm 2.9$ 
	& $47.9\pm 5.2$ \\ 
\cmidrule[0.15pt]{1-5}
\multirow{6}{*}{\shortstack[l]{GNN on\\latent digraph}}
&\emph{Indep. Directed Edges} 
	& $90.1\pm0.5$ 
	& $13.0\pm2.0$ 
	& $56.4\pm6.7$ \\ 
&\emph{Arborescence} &  &  &  \\
&\hspace{2mm} - Neg. Exp.
	& $71.5 \pm 1.4$ 
	& $23.2 \pm 10.2$ 
	& $20.0 \pm 6.0$ \\
&\hspace{2mm} - Gaussian 
	& $\mathbf{95.0 \pm 2.2}$ 
	& $65.3 \pm 3.7$ 
	& $60.8 \pm 7.3$ \\
&\hspace{2mm} - Gumbel 
	& $\mathbf{95.0 \pm 3.0}$ 
	& $\mathbf{75.5 \pm 7.0}$ 
	& $\mathbf{71.9 \pm 12.4}$ \\	
\cmidrule[0.15pt]{2-5}
& Ground Truth Edges 
	& $98.1 \pm 0.1$ 
	& 100 
	& 100 \\
\bottomrule
\end{tabular}
\end{small}
\end{center}
\vspace{-\baselineskip}
\end{table}  \subsection{Learning To Explain (L2X) Aspect Ratings}
With L2X we investigated the effect of the choice of relaxation. We used the BeerAdvocate dataset \citep{mcauley2012learning}, which contains reviews comprised of free-text feedback and ratings for multiple aspects (appearance, aroma, palate, and taste; Fig. \ref{fig:l2x:review}). Each sentence in the test set is annotated with the aspects that it describes, allowing us to define structure recovery metrics. We considered the L2X task of learning a distribution over $k$-subsets of words that best explain a given aspect rating.\footnote{While originally proposed for model interpretability, we used the original aspect ratings. This allowed us to use the sentence-level annotations for each aspect to facilitate comparisons between subset distributions.}  Our model used word embeddings from \citep{lei2016rationalizing} and convolutional neural networks with one (simple) and three (complex) layers to produce a distribution over $k$-hot binary latent masks. Given the latent masks, our model used a convolutional net to make predictions from masked embeddings. We used $k$ in $\{5, 10, 15\}$ and the following \rsmtabbrev{s} for the subset distribution: \{\emph{{Euclid., Cat. Ent., Bin. Ent., E.F. Ent.}}\} \emph{Top $k$} and \emph{Corr. Top $k$}. For baselines, we used bespoke relaxations designed for this task: \emph{{L2X}} \citep{chen2018learning} and \emph{{SoftSub}} \citep{xie2019reparameterizable}. We trained separate models for each aspect using mean squared error (MSE).  

We found that \rsmtabbrev{s} improve over bespoke relaxations (Table \ref{table:l2x_beer_aroma} for aspect aroma, others in App. \ref{supp:sec:addresults}). 
 For unsupervised discovery, we used the sentence-level annotations for each aspect to define ground truth subsets against which precision of the $k$-subsets was measured. \rsmtabbrev{s} tended to select subsets with higher precision across different architectures and cardinalities and achieve modest improvements in MSE. We did not find significant differences arising from the choice of regularizer $f$. Overall, the most structured \rsmtabbrev{}, \emph{Corr. Top $k$}, achieved the lowest MSE, highest precision and improved interpretability: The correlations in the model allowed it to select contiguous words, while subsets from less structured distributions were scattered (Fig. \ref{fig:l2x:review}).

\newcommand{\corrtopk}[1]{\hlmyred{#1}}
\newcommand{\topk}[1]{\hlmyblue{#1}}
\newcommand{\both}[1]{\hlmypurple{#1}}

\begin{table}[t]
  \centering
  \refstepcounter{figure}
  \label{fig:l2x:review}
  \captionsetup{labelformat=andfigure}
  \caption{For $k$-subset selection on aroma aspect, \rsmtabbrev{s} tend to outperform baseline relaxations. Test set MSE ($\times 10^{-2}$) and subset precision (\%) is shown for models selected on valid. MSE. Bottom: \emph{Corr. Top $k$} (red) selects contiguous words while \emph{Top $k$} (blue) picks scattered words.}   \label{table:l2x_beer_aroma}
  \begin{small}
  \adjustbox{max width=\textwidth}{  \begin{tabular}{@{}llcccccc@{}}
  \toprule
   & & \multicolumn{2}{c}{$k=5$} &  \multicolumn{2}{c}{$k=10$} &  \multicolumn{2}{c}{$k=15$} \\
   \cmidrule(lr){3-4} \cmidrule(lr){5-6} \cmidrule(lr){7-8}
  Model & Relaxation & MSE & Subs. Prec. &  MSE & Subs. Prec. &  MSE & Subs. Prec. \\
  \midrule
  \multirow{7}{*}{\shortstack[l]{Simple}}
  & \emph{L2X} \citep{chen2018learning}
  	& $3.6 \pm  0.1$ 
  	& $28.3 \pm 1.7$  
  	& $3.0 \pm  0.1$ 
  	& $25.5 \pm 1.2$ 
  	& $2.6 \pm  0.1$ 
  	& $25.5 \pm 0.4$ \\
  & \emph{SoftSub} \citep{xie2019reparameterizable}
  	& $3.6 \pm  0.1$ 
  	& $27.2 \pm 0.7$ 
  	& $3.0 \pm  0.1$ 
  	& $26.1 \pm 1.1$ 
  	& $2.6 \pm  0.1$ 
  	& $25.1 \pm 1.0$ \\
  \cmidrule[0.15pt]{2-8}
  & \emph{Euclid. Top $k$}
  	& $3.5 \pm 0.1$ 
  	& $25.8 \pm 0.8$ 
  	& $2.8 \pm  0.1$ 
  	& $32.9 \pm 1.2$ 
  	& $2.5 \pm  0.1$ 
  	& $29.0 \pm 0.3$ \\
  & \emph{Cat. Ent. Top $k$}
  	& $3.5 \pm 0.1$ 
  	& $26.4 \pm 2.0$ 
  	& $2.9 \pm  0.1$ 
  	& $32.1 \pm 0.4$ 
  	& $2.6 \pm  0.1$ 
  	& $28.7 \pm 0.5$ \\
  & \emph{Bin. Ent. Top $k$}
  	& $3.5 \pm  0.1$ 
  	& $29.2 \pm 2.0$ 
  	& $2.7 \pm  0.1$ 
  	& $33.6 \pm 0.6$ 
  	& $2.6 \pm  0.1$ 
  	& $28.8 \pm 0.4$ \\
  & \emph{E.F. Ent. Top $k$}
  	& $3.5 \pm  0.1$ 
  	& $28.8 \pm 1.7$ 
  	& $2.7 \pm  0.1$ 
  	& $32.8 \pm 0.5$ 
  	& $2.5 \pm  0.1$ 
  	& $29.2 \pm 0.8$ \\
  \cmidrule[0.15pt]{2-8}
  & \emph{Corr. Top $k$} 
  	& $\mathbf{2.9 \pm 0.1}$ 
  	& $\mathbf{63.1  \pm 5.3}$ 
  	& $\mathbf{2.5 \pm 0.1}$ 
  	& $\mathbf{53.1 \pm 0.9}$ 
  	& $\mathbf{2.4 \pm 0.1}$ 
  	& $\mathbf{45.5 \pm 2.7}$ \\
  \cmidrule[0.15pt]{1-8}
  \multirow{7}{*}{\shortstack[l]{Complex}}
  & \emph{L2X} \citep{chen2018learning}
  	& $2.7 \pm 0.1$ 
  	& $50.5 \pm 1.0$ 
  	& $2.6 \pm 0.1$ 
  	& $44.1 \pm 1.7$ 
  	& $2.4 \pm 0.1$ 
  	& $44.4 \pm 0.9$ \\
  & \emph{SoftSub} \citep{xie2019reparameterizable}
  	& $2.7 \pm 0.1$ 
  	& $57.1 \pm 3.6$ 
  	& $\mathbf{2.3 \pm 0.1}$ 
  	& $50.2 \pm 3.3$ 
  	& $2.3 \pm 0.1$ 
  	& $43.0 \pm 1.1$ \\
  \cmidrule[0.15pt]{2-8}
  & \emph{Euclid. Top $k$}
  	& $2.7 \pm 0.1$ 
  	& $61.3 \pm 1.2$ 
  	& $2.4 \pm 0.1$ 
  	& $52.8 \pm 1.1$ 
  	& $2.3 \pm 0.1$ 
  	& $44.1 \pm 1.2$ \\
  & \emph{Cat. Ent. Top $k$}
  	& $2.7 \pm 0.1$ 
  	& $61.9 \pm 1.2$ 
  	& $\mathbf{2.3 \pm 0.1}$ 
  	& $52.8 \pm 1.0$ 
  	& $2.3 \pm 0.1$ 
  	& $44.5 \pm 1.0$ \\
  & \emph{Bin. Ent. Top $k$}
  	& $2.6 \pm 0.1$ 
  	& $62.1 \pm 0.7$ 
  	& $\mathbf{2.3 \pm 0.1}$ 
  	& $50.7 \pm 0.9$ 
  	& $2.3 \pm 0.1$ 
  	& $44.8 \pm 0.8$ \\
  & \emph{E.F. Ent. Top $k$}  	
  	& $2.6 \pm 0.1$ 
  	& $59.5 \pm 0.9$ 
  	& $\mathbf{2.3 \pm 0.1}$ 
  	& $54.6 \pm 0.6$ 
  	& $2.2 \pm 0.1$ 
  	& $44.9 \pm 0.9$ \\
  \cmidrule[0.15pt]{2-8}
  & \emph{Corr. Top $k$} 
  	& $\mathbf{2.5 \pm 0.1}$ 
  	& $\mathbf{67.9 \pm 0.6}$ 
  	& $\mathbf{2.3 \pm 0.1}$ 
  	& $\mathbf{60.2 \pm 1.3}$ 
  	& $\mathbf{2.1 \pm 0.1}$ 
  	& $\mathbf{57.7 \pm 3.8}$ \\
  \bottomrule
  \end{tabular}}  \end{small}

  \vspace{2pt}
  \setlength{\fboxrule}{\heavyrulewidth}
  \begingroup\fboxsep=0.0025\textwidth
  \fbox{\parbox{0.995\textwidth}{
  \small{
  Pours a \topk{\strut slight tangerine} orange and \topk{\strut straw} yellow. The head is \topk{\strut nice} and bubbly but fades very quickly with a little lacing. \both{\strut Smells}  \corrtopk{\strut{} like Wheat and European hops}, a little yeast in there too. There is some \topk{\strut fruit} in there too, but you have to take a good \topk{\strut whiff} to get it. The taste is of wheat, a bit of malt, and \corrtopk{\strut a little }  \both{\strut fruit}  \corrtopk{ \strut flavour} in there too. Almost feels like drinking \topk{\strut Champagne}, medium mouthful otherwise. Easy to drink, but \topk{\strut not} something I'd be trying every night.
  \begin{center}
  \begin{tabular}{ccccc}
  Appearance: 3.5 & \textbf{Aroma: 4.0} & Palate: 4.5 & Taste: 4.0 & Overall: 4.0
  \end{tabular}
  \end{center}
  }}}
  \endgroup
  \vspace{-\baselineskip}
\end{table}   \section{Conclusion}
We introduced stochastic softmax tricks, which are random convex programs that capture a large class of relaxed distributions over structured, combinatorial spaces.  We designed stochastic softmax tricks for subset selection and a variety of spanning tree distributions. We tested their use in deep latent variable models, and found that they can be used to improve performance and to encourage the unsupervised discovery of true latent structure. There are future directions in this line of work. The relaxation framework can be generalized by modifying the constraint set or the utility distribution at positive temperatures. Some combinatorial objects might benefit from a more careful design of the utility distribution, while others, e.g., matchings, are still waiting to have their tricks designed.
 \section*{Broader Impact}
This work introduces methods and theory that have the potential for improving the interpretability of latent variable models.
While unfavorable consequences cannot be excluded, increased interpretability is generally considered a desirable property of machine learning models.
Given that this is foundational, methodologically-driven research, we refrain from speculating further.
 \section*{Acknowledgements and Disclosure of Funding}
We thank Daniel Johnson and Francisco Ruiz for their time and insightful feedback. We also thank Tamir Hazan, Yoon Kim, Andriy Mnih, and Rich Zemel for their valuable comments. MBP gratefully acknowledges support from the Max Planck ETH Center for Learning Systems. CJM is grateful for the support of the James D. Wolfensohn Fund at the Institute of Advanced Studies in Princeton, NJ. Resources used in preparing this research were provided, in part, by the Sustainable Chemical Processes through Catalysis (Suchcat) National Center of Competence in Research (NCCR), the Province of Ontario, the Government of Canada through CIFAR, and companies sponsoring the Vector Institute.
 \bibliography{refs}
\bibliographystyle{plain}
\newpage
\appendix

\section{Proofs for Stochastic Softmax Tricks}
\label{supp:sec:proofs}
\newcommand{\xmin}{x^{\star}}

\begin{lemma}
  \label{lemma:sameopt}
  Let $\discreteset \subseteq \R^n$ be a finite, non-empty set, $P := \hull(\discreteset)$, and $u \in \R^n$. We have,
  \begin{equation}
    \max_{x \in \discreteset} \, u^Tx = \max_{x \in P} \, u^Tx = \sup_{x \in \relint(P)} \, u^Tx .
\end{equation}
If $\max\nolimits_{x \in \discreteset} \,  u^T x$ has a unique solution $\xmin$, then $\xmin$ is also the unique solution of $\max\nolimits_{x \in P} \,  u^T x$.
\end{lemma}
\begin{proof}
Assume w.l.o.g. that $\discreteset = \{x_1, \ldots, x_m\}$. Let $\xmin \in \arg\max_{x \in \discreteset} u^Tx$.

First, let us consider the linear program over $\discreteset$ vs. $P$. Clearly, $\max_{x \in \discreteset} u^Tx \leq \max_{x \in P} u^Tx$. In the other direction, for any $y \in P$, we can write $y = \sum_{i} \lambda_i x_i$ for $\lambda_i \geq 0$ such that $\sum_i \lambda_i = 1$, and
\begin{equation}
  u^T\xmin = \sum_i \lambda_i u^T\xmin \geq \sum_i \lambda_i u^Tx_i = u^Ty.
\end{equation}
Hence $\max_{x \in \discreteset} u^Tx \geq \max_{x \in P} u^Tx$. Thus $\xmin \in \arg \max_{x \in P} u^Tx$.

Second, let us consider the linear program over $P$ vs. $\relint(P)$. The cases $\xmin \in \relint(P)$ or $u = 0$ are trivial, so assume otherwise. Since $u^T\xmin \geq u^Tx$ for $x \in \relint(P)$, it suffices to show that for all $\epsilon > 0$ there exists $x_{\epsilon} \in \relint(P)$ such that $u^Tx_{\epsilon} > u^T\xmin - \epsilon$. To that end, take $x \in \relint(P)$ and $0 < \lambda < \min(\epsilon, \lVert u \rVert \lVert x - \xmin \rVert)$, and define
\begin{equation}
    x_{\epsilon} := \xmin + \frac{\lambda}{\lVert u \rVert \lVert x - \xmin \rVert}(x - \xmin).
\end{equation}
$x_{\epsilon} \in \relint(P)$ by \citep[][Thm 6.1]{rockafellar1970convex}. Thus, we get
\begin{equation}
    u^Tx_{\epsilon} = u^T\xmin + \lambda \frac{u^T(x - \xmin)}{\lVert u \rVert \lVert x - \xmin \rVert} > u^T\xmin - \epsilon
\end{equation}

Finally, suppose that $\xmin = \arg \max_{x \in \discreteset} u^Tx$ is unique, but $\arg \max_{x \in P} u^Tx$ contains more than just $\xmin$. We will show this implies a contradiction. Let $i^{\star}$ be the index $i \in \{1, \ldots, m\}$ such that $\xmin = x_{i^{\star}}$. Let $y \in \arg \max_{x \in P} u^Tx$ be such that $y \neq \xmin$. Then we may write $y = \sum_i \lambda_i x_i$ for $\lambda_i \geq 0$ such that $\sum_i \lambda_i = 1$. But this leads to a contradiction,
\begin{equation}
    u^T\xmin = \sum_{i \neq i^{\star}} \frac{\lambda_i}{1-\lambda_{i^{\star}}} u^Tx_i < \sum_{i \neq i^{\star}} \frac{\lambda_i}{1-\lambda_{i^{\star}}} u^T\xmin = u^T \xmin.
\end{equation}
\end{proof}

\begin{lemma}
\label{lemma:openinverseset}
Let $P \subseteq \R^n$ be a non-empty convex polytope and $x$ an extreme point of $P$. Define the set,
\begin{equation}
    U(x) = \left\{u \in \R^n : u^Tx > u^Ty, \; \forall y \in P \setminus \{x\}\right\}.
\end{equation}
This is the set of utility vectors of a linear program over $P$ whose argmax is the minimal face $\{x\} \subseteq P$. Then, for all $u \in U(x)$, there exists an open set $O \subseteq U(x)$ containing $u$.
\end{lemma}
\begin{proof}
Let $u \in U(x)$. Let $\{x_1, \ldots, x_m\} \subseteq P$ be the set of extreme points (there are finitely many), and assume w.l.o.g. that $x = x_m$. For each $x_i \neq x_m$ there exists $\epsilon_i > 0$ such that $u^T (x_m - x_i) > \epsilon_i$. Thus, for all $v$ in the open ball $B_{r_i}(u)$  of radius $r_i = \epsilon_i / \lVert x_m - x_i \rVert$ centered at $u$, we have
\begin{equation}
    v^T (x_m - x_i) = u^T (x_m - x_i) + (v - u)^T(x_m - x_i) > u^T (x_m - x_i) - \epsilon_i > 0.
\end{equation}
Define $O = \cap_{i=1}^{m-1} B_{r_i}(u)$. Note, $v^T x_m > v^T x_i$ for all $v \in O, x_i \neq x_m$. Now, let $y \in P \setminus \{x_m\}$. Because $P$ is the convex hull of the $x_i$ \cite[][Thm. 2.9]{bertsimas1997introduction}, we must have
\begin{equation}
   y = \sum_{i=1}^m \lambda_i x_i
\end{equation}
for $\lambda_i \geq 0$, $\sum_{i=1}^m \lambda_i = 1$ with at least one $\lambda_i > 0$ for $i < m$. Thus, for all $v \in O$
\begin{equation}
   v^Tx_m = \sum_{i=1}^m \lambda_i v^T x_m > \sum_{i=1}^m \lambda_i v^T x_i = v^T y.
\end{equation}
This implies that $O \subseteq U(x_m)$, which concludes the proof, as $O$ is open, convex, and contains $u$.
\end{proof}

\begin{lemma}
\label{lemma:cvxconj}
Given a non-empty, finite set $\discreteset \subseteq \R^n$ and a proper, closed, strongly convex function $f : \R^n \to \{\R, \infty\}$ whose domain contains the relative interior of $P := \hull(\discreteset)$, let $f^* = \min_{x \in \R^n} f(x)$, and $\delta_P(x)$ be the indicator function of the polytope $P$,
\begin{equation}
    \delta_P(x) = \begin{cases}
        0 & x \in P\\
        \infty & x \notin P
    \end{cases}.
\end{equation}
For $\temp \geq 0$, define
\begin{align}
    \label{eq:gtemp} g_{\temp}(x) &:= \temp (f(x) - f^*) + \delta_P(x),\\
    \label{eq:convexconjugate} g_{\temp}^*(u) &:= \sup_{x \in \R^n} u^Tx - g_{\temp}(x).
\end{align}
The following are true for $\temp > 0$,
\begin{enumerate}
    \item \eqref{eq:convexconjugate} has a unique solution, $g_{\temp}^*$ is continuously differentiable, twice differentiable a.e., and
    \begin{equation}
        \label{eq:gradcvxconj} \grad g_{\temp}^*(u) = \arg \max_{x \in \R^n} \, u^Tx - g_{\temp}(x).
    \end{equation}
    \item If $\max_{x \in \discreteset} u^Tx$ has a unique solution, then
    \begin{equation}
        \lim_{\temp \to 0^+} \grad g_{\temp}^*(u) = \arg \max_{x \in \discreteset}\,u^T x.
    \end{equation}
\end{enumerate}
\end{lemma}
\begin{proof}
Note, $\relint(P) \subseteq \domain(g_{\temp}) \subseteq P$.
\begin{enumerate}
    \item Since $g_{\temp}$ is strongly convex \citep[][Lem. 5.20]{beck2017fom}, \eqref{eq:convexconjugate} has a unique maximum \citep[][Thm. 5.25]{beck2017fom}. Moreover, $g_{\temp}^*$ is differentiable everywhere in $\R^n$ and its gradient $\grad g_{\temp}^*$ is Lipschitz continuous \citep[][Thm. 5.26]{beck2017fom}. By \citep[][Thm 25.5]{rockafellar1970convex} $\grad g_{\temp}^*$ is a continuous function on $\R^n$. By Rademacher's theorem, $\grad g_{\temp}^*$ is a.e. differentiable. \eqref{eq:gradcvxconj} follows by standard properties of the convex conjugate \citep[][Thm. 23.5, Thm. 25.1]{rockafellar1970convex}.

    \item First, by Lemma \ref{lemma:sameopt},
    \begin{equation}
        g_0^*(u) = \max_{x \in P} u^Tx = \sup_{x \in \relint(P)} u^Tx = \max_{x \in \discreteset} u^Tx.
    \end{equation}
    Since $u$ is such that $u^Tx$ is uniquely maximized over $P$, $g_0^*$ is differentiable at $u$ by \citep[][Thm. 23.5, Thm. 25.1]{rockafellar1970convex}. Again by Lemma \ref{lemma:sameopt} we have
    \begin{equation}
        \label{eq:convexconjugatezero}
        \grad g_0^*(u) = \arg \max_{x \in P} u^Tx = \arg \max_{x \in \discreteset} u^Tx.
    \end{equation}
    Hence, our aim is to show $\lim_{\temp \to 0^+} \grad g_{\temp}^*(u) = \grad g_0^*(u)$. This is equivalent to showing that $\lim_{i \to \infty} \grad g_{\temp_i}^*(u) = \grad g_0^*(u)$ for any $\temp_i > 0$ such that $\temp_i \to 0$. Let $\temp_i$ be such a sequence.

    We will first show that $g_{\temp_i}^*(u) \to g_{0}^*(u)$. For any $y \in \relint(P)$,
\begin{align*}
    \liminf\limits_{i \to \infty} \, g_{\temp_i}^*(u) &= \lim_{i \to \infty} \inf_{j \geq i} \sup_{x \in \R^n} \, u^Tx - g_{\temp_j}(x)\\
    &\geq \lim_{i \to \infty} \inf_{j \geq i} \, u^Ty - g_{\temp_j}(y)\\
    &= u^Ty
\end{align*}
Thus,
\begin{align*}
\liminf\limits_{i \to \infty} g_{\temp_i}^*(u) \geq \sup_{y \in \relint(P)} u^Ty = g_0^*(u)
\end{align*}
Since $\temp(f(x) - f^*) \geq 0$ for all $x \in \R^n$, we also have
\begin{align*}
    \limsup\limits_{i \to \infty} g_{\temp_i}^*(u) &= \limsup\limits_{i \to \infty} \sup_{x \in \R^n} u^Tx - g_{\temp_i}(x)\\
    &\leq \limsup\limits_{i \to \infty} \sup_{x \in P} u^Tx = g_0^*(u).
\end{align*}
Thus $\lim_{i \to \infty} g_{\temp_i}^*(u) = g_0^*(u)$.

By Lemma \ref{lemma:openinverseset}, there exists an open convex set $O$ containing $u$ such that for all $v \in O$, $\grad g_0^*(u) = \arg \max_{x \in P} v^Tx$. Again, $g_0^*$ is differentiable on $O$ \citep[][Thm. 23.5, Thm. 25.1]{rockafellar1970convex}. Using this and the fact that $g_{\temp_i}^*(u) \to g_0^*(u)$, we get $\grad g_{\temp_i}^*(u) \to \grad g_0^*(u)$ \citep[][Thm. 25.7]{rockafellar1970convex}.
\end{enumerate}
\end{proof}

\approximation*
\begin{proof}
For $g_{\temp}^*$ defined in \eqref{eq:convexconjugate}, we have by Lemma \ref{lemma:cvxconj},
\begin{equation}
    \label{eq:redefxtemp}
    X_{\temp} = \arg \max_{x \in P} U^Tx - \temp f(x) = \grad g_{\temp}^*(U).
\end{equation}
If $X$ is a.s. unique, then again by Lemma \ref{lemma:cvxconj}
\begin{align*}
    \proba\left(\lim_{t \to 0^+} X_t = X \right) &= \proba\left(\lim_{t \to 0^+} \grad g_{\temp}^*(U) = \arg \max_{x \in \discreteset} \, U^Tx \right)\\
    &\geq \proba\left(X \text{ is unique} \right)\\
    &= 1
\end{align*}
The last bit of the proof follows from the dominated convergence theorem, since the loss in bounded on $P$ by assumption, so $|\loss(X_{\temp})|$ is surely bounded.
\end{proof}

\relaxation*
\begin{proof}
For $g_{\temp}^*$ defined in \eqref{eq:convexconjugate}, we have by Lemma \ref{lemma:cvxconj},
\begin{equation}
    \label{eq:redefxtemp}
    X_{\temp} = \arg \max_{x \in P} \,  U^T x - \temp f(x) = \grad g_t^*(U).
\end{equation}
Our result follows by the other results of Lemma \ref{lemma:cvxconj}.
\end{proof}

\begin{proposition}
\label{prop:noisedistribution}
If $\proba(U^Ta = 0) = 0$ for all $a \in \R^n$ such that $a \neq 0$, then $X$ in Def. \ref{def:smt} is a.s. unique.
\end{proposition}
\begin{proof}
It suffices to show that for all subsets $S \subseteq \discreteset$ with $|S| > 1$, the event $\{S = \arg \max_{x \in \discreteset} \, U^T x \}$ has zero measure. If $|S| > 1$, then we can pick two distinct points $x_1, x_2 \in S$ with $x_1 \neq x_2$. Now,
\begin{equation}
    \proba\left(S = \arg \max_{x \in \discreteset} \, U^T x \right) = \proba(\forall x \in S, \, U^T x = M) \leq \proba\left(U^T (x_1 - x_2) = 0\right) = 0.
\end{equation}
where $M = \max_{x \in \discreteset} \, U^T x$.
\end{proof}

\begin{proposition}
\label{prop:convexposition}
Let $\discreteset \subseteq \R^n$ be a non-empty finite set. If $\discreteset$ is convex independent, i.e., for all $x \in \discreteset$, $x \notin \hull(\discreteset \setminus \{x\})$, then $\discreteset$ is the set of extreme points of $\hull(\discreteset)$. In particular, any non-empty set of binary vectors $\discreteset \subseteq \{0, 1\}^n$ is convex independent and thus the set of extreme points of $\hull(\discreteset)$.
\end{proposition}
\begin{proof}
Let $\discreteset = \{x_1, \ldots, x_m\}$. The fact that the extreme points of $\hull(\discreteset)$ are in $\discreteset$ is trivial. In the other direction, it is enough to show that $x_m$ is an extreme point. Assume $x_m \in \discreteset$ is not an extreme point of $\hull(\discreteset)$. Then by definition, we can write $x_m = \lambda y + (1-\lambda) z$ for $y,z \in \hull(\discreteset)$, $\lambda \in (0,1)$ with $y \neq x_m$ and $z \neq x_m$. Then, we have that
\begin{equation}
    \label{eq:convexpositioncontradiction}
    x_m = \sum_{i=1}^{m-1} \frac{\lambda \alpha_i + (1-\lambda) \beta_i}{ 1 - \lambda \alpha_m - (1-\lambda) \beta_m} x_i
\end{equation}
for some sequences $\alpha_i, \beta_i \geq 0$ such that $\sum_{i=1}^m \alpha_i = \sum_{i=1}^m \beta_i = 1$ and $\alpha_m, \beta_m < 1$. This is clearly a contradiction of our assumption that $x_m \notin \hull(\discreteset \setminus \{x_m\})$, since the weights in the summation \eqref{eq:convexpositioncontradiction} sum to unity. This implies that $\discreteset$ are the extreme points of $\hull(\discreteset)$.

Let $\discreteset \subseteq \{0,1\}^n$. It is enough to show that $x_m \notin \hull(\{x_1, \ldots, x_{m-1}\})$. Assume this is not the case. Let $c = x_m - 1/2 \in \R^n$, and note that $c^T x_i < c^T x_m$ for all $i \neq m$ when $x_i$ are distinct binary vectors. But, this leads to a contradiction. By assumption we can express $x_m$ as a convex combination of $x_1, \ldots, x_{m-1}$. Thus, there exists $\lambda_i \geq 0$ such that $\sum_{i=1}^{m-1} \lambda_i = 1$, and
\begin{equation}
    c^T x_m = \sum_{i=1}^{m-1} \lambda_i c^T x_i < \sum_{i=1}^{m-1} \lambda_i c^T x_m = c^T x_m.
\end{equation}
\end{proof}
 \section{An Abbreviated Field Guide to Stochastic Softmax Tricks}
\label{supp:sec:fieldguide}

\subsection{Introduction}
\label{supp:sec:fieldguide:intro}
\paragraph{Overview.}  This is a short field guide to some stochastic softmax tricks (\rsmtabbrev{s}) and their associated stochastic argmax tricks (\smtabbrev{s}). There are many potential \rsmtabbrev{s} not discussed here. We assume throughout this Appendix that readers are completely familiar with main text and its notation; we do not review it. In particular, we follow the problem definition and notation of Section \ref{sec:problemstmt}, the definition and notation of \smtabbrev{s} in Section \ref{sec:smts}, and the definition and notation of \rsmtabbrev{s} in Section \ref{sec:rsmts}.

This field guide is organized by the abstract set $\abstractset$. For each $\abstractset$, we identify an appropriate set $\discreteset \subseteq \R^n$ of structured embeddings. We discuss utility distributions used in the experiments. In some cases, we can provide a simple, ``closed-form'', categorical sampling process for $X$, i.e., a generalization of the Gumbel-Max trick. We also cover potential relaxations used in the experiments. In the remainder of this introduction, we introduce basic concepts that recur throughout the field guide.

\paragraph{Notation.} Given a finite set $S$, the indicator vector $x_T$ of a subset $T \subseteq S$ is the binary vector $x_T := (x_s)_{s \in S}$ such that $x_s = 1$ if $s \in T$ and $x_s = 0$ if $s \notin T$. For example, given an graph $G = (V, E)$, let $T$ be the edges of a spanning tree (ignoring the direction of edges). The indicator vector $x_T$ of $T$ is the vector $(x_e)_{e \in E}$ with $x_e = 1$ if $e$ is in the tree and $x_e = 0$ if $e$ is not.

$X \sim \mathcal{D}(\theta, Y)$ means that $X$ is distributed according to $\mathcal{D}$, which takes arguments $\theta$ and $Y$. Unless otherwise stated, $X$ is conditionally independent from all other random variables given $\theta, Y$. For multidimensional $U \in \R^n$, we use the same notation:
\begin{align}
  U \sim \exponential(\lambda) \iff U_i \sim \exponential(\lambda_i) \text{ independent}
\end{align}
Given $A \subseteq \{1, \ldots, n\}$ and $\lambda_i \in (0, \infty]$ for $0 < i \leq n$, the following notation,
\begin{align}
  K \sim \lambda_i \mathbf{1}_A(i),
\end{align}
means that $K$ is a random integer selected from $A$ with probability proportional to $\lambda_i$. If any $\lambda_i = \infty$, then we interpret this as a uniform random integer from the integers $i \in A$ with $\lambda_i = \infty$.

\paragraph{Basic properties of exponentials and Gumbels.} The properties of Gumbels and exponentials are central to \smtabbrev{s} that have simple descriptions for the marginal $p_{\theta}$. We review the important ones here. These are not new; many have been used for more elaborate algorithms that manipulate Gumbels \citep[e.g.,][]{maddison2014astarsamp}.

A Gumbel random variable $G \sim \Gumbel(\theta)$ for $\theta \in \R^n$ is a location family distribution, which can be simulated using the identity
\begin{equation}
  G \overset{d}{=} \theta - \log(-\log U),
\end{equation}
for $U \sim \uniform(0,1)$. An exponential random variable $E \sim \exponential(\lambda)$ for rate $\lambda > 0$ can be simulated using the identity
\begin{equation}
  E \overset{d}{=} -\log U / \lambda,
\end{equation}
for $U \sim \uniform(0,1)$. Any result for exponentials immediately becomes a result for Gumbels, because they are monotonically related:
\begin{proposition}
  \label{prop:gumbelexprelationship}
    If $E \sim \exponential(\lambda)$, then $-\log E \sim \Gumbel(\log \lambda)$.
\end{proposition}
\begin{proof}
  If $U \sim \uniform(0,1)$, then $-\log E \overset{d}{=} -\log(-\log U) + \log\lambda \sim \Gumbel(\log \lambda)$.
\end{proof}
Although we prove results for exponentials, using their monotonic relationship, all of these results have analogs from Gumbels.

The properties of exponentials are summarized in the following proposition.
\begin{proposition}
  \label{prop:exponentialmintrick}
  If $E_i \sim \exponential(\lambda_i)$ independent for $\lambda_i > 0$ and $i \in \{1, \ldots, n\}$, then
  \begin{enumerate}
    \item $\arg \min_i E_i \sim \lambda_i$,
    \item $\min_i E_i \sim \exponential(\sum\nolimits_{i=1}^n \lambda_i)$,
    \item $\min_i E_i$ and $\arg \min_i E_i$ are independent,
    \item Given $K = \arg \min_i E_i$ and $E_K = \min_i E_i$, $E_i$ for $i \neq K$ are conditionally, mutually independent; exponentially distributed with rates $\lambda_i$; and truncated to be larger than $E_K$.
  \end{enumerate}
\end{proposition}
\begin{proof}
  The joint density of $K = \arg \min_i E_i$ and $E_i$ is given by $\prod_{i=1}^n \lambda_i \exp(-\lambda_i e_i) \mathbf{1}_{x \geq e_k}(e_i)$. Manipulating this joint, we can see that
  \begin{equation}
    \begin{aligned}
      \label{eq:bigjoint}
    \prod_{i=1}^n \lambda_i &\exp(-\lambda_i e_i) \mathbf{1}_{x \geq e_k}(e_i)\\
    &= \lambda_k \exp(-\lambda_k e_k) \prod_{i \neq k} \lambda_i \exp(-\lambda_i e_i) \mathbf{1}_{x \geq e_k}(e_i)\\
    &= \frac{\lambda_k}{\sum_{i=1}^n \lambda_i} \left(\sum_{i=1}^n \lambda_i\right) \exp(-\lambda_k e_k) \prod_{i \neq k} \lambda_i \exp(-\lambda_i e_i) \mathbf{1}_{x \geq e_k}(e_i)\\
    &= \left[\frac{\lambda_k}{\sum_{i=1}^n \lambda_i} \right] \left[ \left(\sum_{i=1}^n \lambda_i\right) \exp\left(-\sum_{i=1}^n \lambda_i e_k\right) \right] \left[  \prod_{i \neq k}  \frac{\lambda_i\exp(-\lambda_i e_i)}{\exp(-\lambda_i e_k)} \mathbf{1}_{x \geq e_k}(e_i) \right]
    \end{aligned}
  \end{equation}
  While hard to parse, this manipulation reveals the all of the assertions of the proposition.
\end{proof}

Prop. \ref{prop:exponentialmintrick} has a couple of corollaries. First, subtracting the minimum exponential from a collection {only} affects the distribution of the minimum, leaving the distribution of the other exponentials unchanged.
\begin{corollary}
  \label{cor:exponentialmintrick}
  If $E_i \sim \exponential(\lambda_i)$ independent for $\lambda_i > 0$ and $i \in \{1, \ldots, n\}$, then $E_i - \min_i E_i$ are mutually independent and
  \begin{align}
    E_i - \min_i E_i \sim \begin{cases}
      \exponential(\lambda_i) & i \neq K\\
      0 & i = K\\
    \end{cases},
  \end{align}
  where $K = \arg \min_i E_i$.
\end{corollary}
\begin{proof}
  Consider the change of variables $e_i' = e_i - e_k$ in the joint \eqref{eq:bigjoint}. Each of the terms in the right hand product over $i \neq k$ of \eqref{eq:bigjoint} are transformed in the following way
 \begin{equation}
   \frac{\lambda \exp(- \lambda_i (e_i' + e_k))}{\exp(-\lambda_i e_k)} \mathbf{1}_{x \geq e_k}(e_i' + e_k) \longrightarrow \lambda_i \exp(-\lambda_i e_i')
 \end{equation}
 This is essentially the memoryless property of exponentials. Thus, the $E_i' = E_i - E_K$ for $i \neq K$ are distributed as exponentials with rate $\lambda_i$ and mutually independent. $E_K'$ is the constant 0, which is independent of any random variable. Our result follows.
\end{proof}
Second, the process of sorting the collection $E_i$ is equivalent to sampling from $\{1, \ldots, n\}$ without replacement with probabilities proportional to $\lambda_i$.
\begin{corollary}
  \label{cor:exponentialsort}
  Let $E_i \sim \exponential(\lambda_i)$ independent for $\lambda_i > 0$ and $i \in \{1, \ldots, n\}$. Let $\argsort_x : \{1, \ldots, n\} \to \{1, \ldots, n\}$ be the argsort permutation of $x \in \R^n$, i.e., the permutation such that $x_{\argsort_x(i)}$ is in non-decreasing order. We have
  \begin{equation}
    \proba(\argsort_{E} = \sigma) = \prod_{i = 1}^{n} \frac{\lambda_{\sigma(i)}}{\sum_{j = i}^{n} \lambda_{\sigma(j)}}
  \end{equation}
  Given $\argsort_{E} = \sigma$, the sorted vector $E_{\sigma} = (E_{\sigma(i)})_{i=1}^n$ has the following distribution,
  \begin{equation}
    \begin{aligned}
      \label{eq:exponentialchain}
      &E_{\sigma(1)} \sim \exponential\left( \sum\nolimits_{j=1}^n \lambda_{\sigma(j)} \right)\\
      &E_{\sigma(i)} - E_{\sigma(i-1)} \sim \exponential\left(\sum\nolimits_{j=i}^n \lambda_{\sigma(j)} \right)
    \end{aligned}
  \end{equation}
\end{corollary}
\begin{proof}
  This follows after repeated, interleaved uses of Cor. \ref{cor:exponentialmintrick} and Prop. \ref{prop:exponentialmintrick}.
\end{proof} 
\subsection{Element Selection}

\paragraph{One-hot binary embeddings.} Given a finite set $\abstractset$ with $|\abstractset| = n$, we can associate each $y \in \abstractset$ with a one-hot binary embedding. Let $\discreteset \subseteq \R^n$ be the following set of one-hot embeddings,
\begin{equation}
    \discreteset = \left\{x \in \{0, 1\}^n \, \middle | \, \sum_i x_i = 1\right\}.
\end{equation}
For $u \in \R^n$, a solution to the linear program $\xmin \in \arg \max_{x \in \discreteset} u^Tx$ is given by setting $\xmin_k = 1$ for $k \in \arg \max_i u_i$ and $\xmin_k = 0$ otherwise.

\paragraph{Random Utilities.} If $U \sim \Gumbel(\theta)$, then $X \sim \exp(\theta_i)$. This is known as the Gumbel-Max trick \citep{luce1959individual, maddison2014astarsamp}, which follows from Props. \ref{prop:gumbelexprelationship} and \ref{prop:exponentialmintrick}.

\paragraph{Relaxtions.} If $f(x) = \sum_i x_i \log x_i$, then the \rsmtabbrev{} solution $X_t$ is given by
\begin{equation}
  X_t = \left(\frac{\exp(U_i/t)}{\sum_{j=1}^n \exp(U_j/t)}\right)_{i=1}^n.
\end{equation}
In this case, the categorical entropy relaxation and the exponential family relaxation coincide. This is known as the Gumbel-Softmax trick when $U \sim \Gumbel(\theta)$ \citep{maddison2016concrete, jang2016categorical}. If $f(x) = \lVert x \rVert^2 / 2$, then $X_t$ can be computed using the \texttt{sparsemax} operator \citep{martins2016softmax}. In analogy, we name this relaxation with $U \sim \Gumbel(\theta)$ the Gumbel-Sparsemax trick. 
 
\subsection{Subset Selection}

\paragraph{Binary vector embeddings.} Given a finite set $S$ with $|S| = n$, let $\abstractset$ be the set of all subsets of $S$, i.e., $\abstractset = 2^S := \left\{y \subseteq S \right\}$. The indicator vector embeddings of $\abstractset$ is the set,
\begin{equation}
    \discreteset = \{x_y : y \in 2^S\} = \{0,1\}^{|S|}
\end{equation}
For $u \in \R^n$, a solution to the linear program $\xmin \in \arg \max_{x \in \discreteset} u^Tx$ is given by setting $\xmin_i = 1$ if $u_i > 0$ and $\xmin_i = 0$ otherwise, for all $i \leq n$.

\paragraph{Random utilities.} If $U \sim \Logistic(\theta)$, then $X_i \sim \Bernoulli(\sigmoid(\theta_i))$ for all $i \leq n$, where $\sigmoid(\cdot)$ is the sigmoid function. This corresponds to an application of the Gumbel-Max trick independently to each element in $S$. $U \sim \Logistic(\theta)$ has the same distribution as $\theta + \log U' - \log(1-U')$ for $U' \sim \uniform(0,1)$.

\paragraph{Relaxations.} For this case, the exponential family and the binary entropy relaxation, where $f(x) = \sum\nolimits_{i=1}^n x_i \log (x_i) + (1-x_i) \log(1-x_i)$, coincide. The \rsmtabbrev{} solution $X_\temp$ is given by
\begin{equation}
  X_t = \left(\sigmoid(U_i/t)\right)_{i=1}^n
\end{equation}
where $\sigmoid(\cdot)$ is the sigmoid function. For the categorical entropy relaxation with $f(x) = \sum\nolimits_{i=1}^n x_i \log(x_i)$, the \rsmtabbrev{} solution is given by $X_t = \left(\min(1, \exp(U_i/t))\right)_{i=1}^n$ \citep{blondel2019structured}. 
\subsection{$k$-Subset Selection}

\paragraph{$k$-hot binary embeddings.} Given a finite set $S$ with $|S| = n$, let $\abstractset$ be the set of all subsets of $S$ with cardinality $1 \leq k < n$, i.e., $\abstractset = \left\{y \subseteq S \, \middle | \,  |y| = k \right\}$. The indicator vector embeddings of $\abstractset$ is the set,
\begin{equation}
    \discreteset = \left\{x_y : y \subseteq S,\ |y| = k\right\}
\end{equation}
For $u \in \R^n$, let $\argtopk{u}$ be the operator that returns the indices of the $k$ largest values of $u$. For $u \in \R^n$, a solution to the linear program $\xmin \in \arg \max_{x \in \discreteset} u^Tx$ is given by setting $\xmin_i = 1$ for $i \in \argtopk{u}$ and $\xmin_i = 0$ otherwise. 

\paragraph{Random utilities.} If $U \sim \Gumbel(\theta)$, this induces a Plackett-Luce model \citep{luce1959individual}\citep{plackett1975analysis} over the indices that sort $U$ in descending order. In particular, $X$ may be sampled by sampling $k$ times without replacement from the set $\{1, \ldots, n\}$ with probabilities proportional to $\exp(\theta_i)$, setting the sampled indices of $X$ to 1, and the rest to 0 \citep{kool2020gumbeltopk}. This can be seen as a consequence of Cor. \ref{cor:exponentialsort}. 
 
\paragraph{Relaxations.}  For the Euclidean relaxation with $f(x) = \lVert x \rVert^2 / 2$, $X_{\temp}$ we computed $X_\temp$ using a bisection method to solve the constrained quadratic program, but note that other algorithms are available \citep{blondel2019structured}. For the categorical entropy relaxation with $f(x) = \sum\nolimits_{i=1}^n x_i \log(x_i)$, the SST solution $X_\temp$ can be computed efficiently using the algorithm described in \citep{martins2017learning}. For the binary entropy relaxation with $f(x) = \sum\nolimits_{i=1}^n x_i \log (x_i) + (1-x_i) \log(1-x_i)$, the SST solution can be computed using the algorithm in \citep{amos2019limited}. Finally, for the exponential family relaxation, the SST solution can be computed using dynamic programming as described in \citep{tarlow2012fast}. 
\subsection{Correlated $k$-Subset Selection}

\paragraph{Correlated $k$-hot binary embeddings.} Given a finite set $S$ with $|S| = n$, let $\abstractset$ be the set of all subsets of $S$ with cardinality $1 \leq k < n$, i.e., $\abstractset = \left\{y \subseteq V \, \middle | \,  |y| = k \right\}$. We can associate each $y \in \abstractset$ with a $(2n-1)$-dimensional binary embedding with a $k$-hot cardinality constraint on the first $n$ dimensions and a constraint that the $n-1$ dimensions indicate correlations between adjacent dimensions in the first $n$, i.e. the vertices of the correlation polytope of a chain \citep[][Ex. 3.8]{wainwright2008graphical} with an added cardinality constraint \citep{mezuman2013tighter}. Let $\discreteset \subseteq \R^n$ be the set of all such embeddings,
\begin{align}
    \discreteset = \left\{x \in \{0, 1\}^{2n-1} \, \middle | \, \sum\nolimits_{i=1}^n x_i = k; \ x_i = x_{i-n}x_{i-n+1} \text{ for all } n < i \leq 2n-1\right\}.
\end{align}
For $u \in \R^n$, a solution to the linear program $\xmin \in \arg \max_{x \in \discreteset} u^Tx$ can be computed using dynamic programming \citep{tarlow2012fast, mezuman2013tighter}.

\paragraph{Random utilities.} In our experiments for correlated $k$-subset selection we considered Gumbel unary utilities with fixed pairwise utilities. This is, we considered $U_i \sim \Gumbel(\theta_i)$ for $i \leq n$ and $U_i = \theta_i$ for $n < i \leq 2n-1$. 

\paragraph{Relaxations.}  The exponential family relaxation for correlated $k$-subsets can be computed using dynamic programming as described in \citep{tarlow2012fast, mezuman2013tighter}.  
\subsection{Perfect Bipartite Matchings}

\paragraph{Permutation matrix embeddings.} Given a complete bipartite graph $K_{n,n}$, let $\abstractset$ be the set of all perfect matchings. We can associate each $y \in \abstractset$ with a permutation matrix and let $\discreteset$ be the set of all such matrices,
\begin{equation}
	\discreteset = \left\{x \in \{0, 1\}^{n\times n} \, \middle | \, 
	\text{ for all }1 \leq i, j \leq n, \, \sum_i x_{ij} = 1, \, \sum_j x_{ij} = 1  \right\}.
\end{equation}
For $u \in \R^{n\times n}$, a solution to the linear program $\xmin \in \arg \max_{x \in \discreteset} u^Tx$ can be computed using  the Hungarian method \citep{kuhn1955hungarian}.

\paragraph{Random utilities.} Previously, \citep{mena2018learning} considered $U \sim \Gumbel(\theta)$ and \citep{grover2018stochastic} uses correlated Gumbel-based utilities that induce a Plackett-Luce model \citep{luce1959individual}\citep{plackett1975analysis}. 

\paragraph{Relaxations.} For the categorical entropy relaxation with $f(x) = \sum\nolimits_{i=1}^n x_i \log(x_i)$, the \rsmtabbrev{} solution $X_\temp$ can be computed using the Sinkhorn algorithm \citep{sinkhorn1967concerning}. When choosing Gumbel utilities, this recovers Gumbel-Sinkhorn \citep{mena2018learning}. This relaxation can also be used to relax the Plackett-Luce model, if combined with the utility distribution in \citep{grover2018stochastic}. 
\subsection{Undirected Spanning Trees}

\paragraph{Edge indicator embeddings.} Given a undirected graph $G = (V, E)$, let $\abstractset$ be the set of spanning trees of $G$ represented as subsets $T \subseteq E$ of edges. The indicator vector embeddings of $\abstractset$ is the set,
\begin{equation}
  \discreteset = \cup_{T \in \abstractset} \{x_T\}.
\end{equation}
We assume that $G$ has at least one spanning tree, and thus $\discreteset$ is non-empty. A linear program over $\discreteset$ is known as a maximum weight spanning tree problem. It is efficiently solved by the Kruskal's algorithm \citep{kruskal1956shortest}.

\paragraph{Random utilities.} In our experiments, we used $U \sim \Gumbel(\theta)$. In this case, there is a simple, categorical sampling process that described the distribution over $X$. 

The sampling process follows Kruskal's algorithm \citep{kruskal1956shortest}. The steps of Kruskal's algorithm are as follows: sort the list of edges $e$ in non-increasing order according to their utilities $U_e$, greedily construct a tree by adding edges to $T$ as long as no cycles are created, and return the indicator vector $x_T$. Using Cor. \ref{cor:exponentialsort} and Prop. \ref{prop:gumbelexprelationship}, for Gumbel utilities this is equivalent to the following process: sample edges $e$ without replacement with probabilities proportional to $\exp(\theta_e)$, add edges $e$ to $T$ in the sampled order as long as no cycles are created, and return the indicator vector $x_T$.

\paragraph{Relaxations.} The exponential family relaxation for spanning trees can be computed using Kirchhoff's Matrix-Tree Theorem. Here we present a quick informal review. Consider an exponential family with natural parameters $u \in \R^{|E|}$ over $\discreteset$ such that the probability of $x \in \discreteset$ is proportional to $\exp(u^T x)$. Define the weights,
\begin{equation}
  w_{ij} = \begin{cases}
  \exp(u_{e}) & \text{if } i \neq j \text{ and } \exists \, e \in E \text{ connecting nodes } i \text{ and } j\\
  0 & \text{otherwise}
\end{cases}.
\end{equation}
Consider the graph Laplacian $L \in \R^{|V| \times |V|}$ defined by
\begin{equation}
  L_{ij} = \begin{cases}
    \sum_{k \neq j} w_{kj} & \text{if } i = j\\
    - w_{ij} & \text{if } i \neq j
\end{cases}
\end{equation}
Let $L^{k,k}$ be the submatrix of $L$ obtained by deleting the $k$th row and $k$th column. The Kirchhoff Matrix-Tree Theoreom states that
\begin{equation}
  \log \det L^{k,k} = \log \left(\sum_{T \in \abstractset} \exp\left(u^T x_T\right)\right).
\end{equation}
 \citep[][p. 14]{tutte1984graphtheory} for a reference. We can use this to compute the marginals of the exponential family via its derivative \citep{wainwright2008graphical}. In particular,
\begin{equation}
  \mu(u) := \left(\frac{\partial \log \det L^{k,k}}{\partial u_e}\right)_{e \in E} = \sum_{T \in \abstractset}\frac{x_T \exp\left(u^T x_T\right)}{\sum_{T' \in \abstractset} \exp\left(u^T x_{T'}\right)}.
\end{equation}
These partial derivatives can be computed in the standard auto-diff libraries. All together, we may define the exponential family relaxation via $X_t = \mu(U/t)$.
 
\subsection{Rooted, Directed Spanning Trees}

\paragraph{Edge indicator embeddings.} Given a directed graph $G = (V, E)$, let $\abstractset$ be the set of $r$-arborescences for $r \in V$.  An $r$-arborescence is a subgraph of $G$ that is a spanning tree if the edge directions are ignored and that has a directed path from $r$ to every node in $V$. Let $x_{T} := (x_e)_{e \in E}$ be the indicator vector of an $r$-arborescence with edges $T \subseteq E$. Define the set $\mathcal{T}(r)$ of $r$-arborescences of $G$. The indicator vector embeddings of $\abstractset$ is the set,
\begin{equation}
  \discreteset = \cup_{T \in \mathcal{T}(r)} \{x_T\}.
\end{equation}
We assume that $G$ has at least one $r$-arborescence, and thus $\discreteset$ is non-empty. A linear program over $\discreteset$ is known as a maximum weight $r$-arborescence problem. It is efficiently solved by the Chu-Liu-Edmonds algorithm (CLE) \citep{chu1965shortest, edmonds1967optimum}, see Alg. \ref{alg:maxarbor} for an implementation by \citep{kleinberg2006algorithmdesign}.

\paragraph{Random utilities.} In the experiments, we tried $U \sim \Gumbel(\theta)$, $-U \sim \exponential(\theta)$ with $\theta > 0$, and $U \sim \Normal(\theta, 1)$. As far as we know $X$ does not have any particularly simple closed-form categorical sampling process in the cases $U \sim \Gumbel(\theta)$ or $U \sim \Normal(\theta, 1)$.

In contrast, for negative exponential utilities $-U \sim \exponential(\theta)$, $X$ can be sampled using the sampling process given in Alg. \ref{alg:arborcatred}. In some sense, Alg. \ref{alg:arborcatred} is an elaborate generalization of the Gumbel-Max trick to arborescences.

We will argue that Alg. \ref{alg:arborcatred} produces the same distribution over its output as Alg. \ref{alg:maxarbor} does on negative exponential $U_e$. To do this, we will argue that joint distribution of the sequence of edge choices (lines 2-4 colored red in Alg. \ref{alg:maxarbor}), after integrating out $U$, is given by lines 2-4 (colored blue) of Alg. \ref{alg:arborcatred}. Consider the first call to CLE: all $U_e$ are negative and distinct almost surely, for each node $v \neq r$ the maximum utility edge is picked from the set of entering edges $E_v$, and all edges have their utilities modified by subtracting the maximum utility. The argmax of $U_e$ over $E_v$ is a categorical random variable with mass function proportional to the rates $\lambda_e$, and it is independent of the max of $U_e$ over $E_v$ by Prop. \ref{prop:exponentialmintrick}. By Cor. \ref{cor:exponentialmintrick}, the procedure of modifying the utilities leaves the distribution of all unpicked edges invariant and sets the utility of the argmax edge to $0$. Thus, the distribution of $U'$ passed one level up the recursive stack is the same as $U$ with the exception of a randomly chosen subset of utilities $U_e'$ whose rates have been set to $\infty$. The equivalence in distribution between Alg. \ref{alg:maxarbor} and Alg. \ref{alg:arborcatred} follows by induction.

\begin{figure}[t]
\begin{minipage}[t]{0.49\textwidth}
\null
  \begin{algorithm}[H]
  \SetKwInput{KwIn}{Init}
   \KwIn{graph $G$, node $r$, $U_e \in \R$, $T=\emptyset$;}
   \ForEach{node $v \neq r$}{
    {\color{myred}
    $E_v = \{$edges entering $v\}$\;
    $U_e^{\prime} = U_e - \max_{e \in E_v} U_e,$ $\forall e \in E_v$\;
    Pick $e \in E_v$ s.t. $U_e^{\prime} = 0$; $T = T \cup\{e\}$\;}
   }
   \lIf{$T$ is an arborescence}{\Return $x_T$}
   \Else(there is a directed cycle $C \subseteq T$){
   Contract $C$ to supernode, form graph $G'$\;
    Recurse on $(G', r, U')$ to get arbor. $T'$\;
    Expand $T'$ to subgraph of $G$ and add\\
    all but one edge of $C$; \Return{$x_{T'}$}\;
   }
   \caption{Maximum $r$-arborescence \citep{kleinberg2006algorithmdesign}}
   \label{alg:maxarbor}
  \end{algorithm}
\end{minipage}\hfill
\begin{minipage}[t]{0.49\textwidth}
\null
\begin{algorithm}[H]
  \SetKwInput{KwIn}{Init}
  \KwIn{graph $G$, node $r$, $\lambda_e > 0$, $T=\emptyset$;}
 \ForEach{node $v \neq r$}{
  {\color{myblue}
  $E_v = \{$edges entering $v\}$\;
  Sample $e \sim \lambda_e \mathbf{1}_{E_v}(e)$; $T = T \cup \{e\}$\;
  $\lambda_a' = \lambda_a$ if $a \neq e$ else $\infty$, $\forall a \in E_v$\;}
 }
 \lIf{$T$ is an arborescence}{\Return $x_T$}
 \Else(there is a directed cycle $C \subseteq T$){
  Contract $C$ to supernode, form graph $G'$\;
  Recurse on $(G', r, \lambda')$ to get arbor. $T'$\;
  Expand $T'$ to subgraph of $G$ and add\\
  all but one edge of $C$; \Return{$x_{T'}$}\;
 }
 \caption{Equiv. for neg. exp. $U$}
 \label{alg:arborcatred}
\end{algorithm}
\end{minipage}
\caption{Alg. \ref{alg:maxarbor} and Alg. \ref{alg:arborcatred} have the same output distribution for negative exponential $U$, i.e., Alg. \ref{alg:arborcatred} is an equivalent categorical sampling process for $X$. Alg. \ref{alg:maxarbor} computes the maximum point of a stochastic $r$-arborescence trick with random utilities $U_e$ \citep{kleinberg2006algorithmdesign}. When $-U_e \sim \exponential(\lambda_e)$, it has the same distribution as Alg. \ref{alg:arborcatred}. Alg. \ref{alg:arborcatred} samples a random $r$-arborescence given rates $\lambda_e > 0$ for each edge. Both Algs. assume that $G$ has at least one $r$-arbor. Color indicates the main difference.}\label{fig:arborescence}
\end{figure}

\paragraph{Relaxations.} The exponential family relaxation for $r$-arborescences can be computed using the directed version of Kirchhoff's Matrix-Tree Theorem. Here we present a quick informal review. Consider an exponential family with natural parameters $u \in \R^{|E|}$ over $\discreteset$ such that the probability of $x \in \discreteset$ is proportional to $\exp(u^T x)$. Define the weights,
\begin{equation}
  w_{ij} = \begin{cases}
  \exp(u_{e}) & \text{if } i \neq j \text{ and } \exists \, e \in E \text{ from node } i \to j\\
  0 & \text{otherwise}
\end{cases}.
\end{equation}
Consider the graph Laplacian $L \in \R^{|V| \times |V|}$ defined by
\begin{equation}
  L_{ij} = \begin{cases}
    \sum_{k \neq j} w_{kj} & \text{if } i = j\\
    - w_{ij} & \text{if } i \neq j
\end{cases}
\end{equation}
Let $L^{r,r}$ be the submatrix of $L$ obtained by deleting the $r$th row and $r$th column. The result by Tutte \citep[][p. 140]{tutte1984graphtheory} states that
\begin{equation}
  \log \det L^{r,r} = \log \left(\sum_{T \in \mathcal{T}(r)} \exp\left(u^T x_T\right)\right)
\end{equation}
We can use this to compute the marginals of the exponential family via its derivative \citep{wainwright2008graphical}. In particular,
\begin{equation}
  \mu(u) := \left(\frac{\partial \log \det L^{r,r}}{\partial u_e}\right)_{e \in E} = \sum_{T \in \mathcal{T}(r)}\frac{x_T \exp\left(u^T x_T\right)}{\sum_{T' \in \mathcal{T}(r)} \exp\left(u^T x_{T'}\right)}.
\end{equation}
These partial derivatives can be computed in the standard auto-diff libraries. All together, we may define the exponential family relaxation via $X_t = \mu(U/t)$.
  \section{Additional Results}
\label{supp:sec:addresults}
\begin{table}[t]
  \centering
  \caption{For $k$-subset selection on appearance aspect, \rsmtabbrev{s} select subsets with high precision and outperform baseline relaxations. Test set MSE ($\times 10^{-2}$)  and subset precision (\%) is shown for models selected on valid. MSE.}  \label{table:l2x_beer_appearance}
  \begin{small}
  \adjustbox{max width=\textwidth}{  \begin{tabular}{@{}llcccccc@{}}
  \toprule
   & & \multicolumn{2}{c}{$k=5$} &  \multicolumn{2}{c}{$k=10$} &  \multicolumn{2}{c}{$k=15$} \\
   \cmidrule(lr){3-4} \cmidrule(lr){5-6} \cmidrule(lr){7-8}
  Model & Relaxation & MSE & Subs. Prec. &  MSE & Subs. Prec. &  MSE & Subs. Prec. \\
  \midrule
  \multirow{7}{*}{\shortstack[l]{Simple}}
  & \emph{L2X} \citep{chen2018learning}
  	& $3.1 \pm 0.1$ 
  	& $48.7 \pm 0.6$ 
  	& $2.6 \pm 0.1$ 
  	& $41.9 \pm 0.6$ 
  	& $2.5 \pm 0.1$ 
  	& $38.6 \pm 1.5$ \\
  & \emph{SoftSub} \citep{xie2019reparameterizable}
  	& $3.2 \pm 0.1$ 
  	& $43.9 \pm 1.1$ 
  	& $2.7 \pm 0.1$ 
  	& $41.9 \pm 2.1$ 
  	& $2.5 \pm 0.1$ 
  	& $38.0 \pm 2.4$ \\
  \cmidrule[0.15pt]{2-8}
  & \emph{Euclid. Top $k$}
  	& $3.0 \pm 0.1$ 
  	& $49.4 \pm 1.7$ 
  	& $2.6 \pm 0.1$ 
  	& $48.8 \pm 1.2$ 
  	& $2.4 \pm 0.1$ 
  	& $42.9 \pm 1.0$ \\
  & \emph{Cat. Ent. Top $k$}
  	& $3.0 \pm 0.1$ 
  	& $53.2 \pm 1.7$ 
  	& $2.6 \pm 0.1$ 
  	& $46.3 \pm 1.9$ 
  	& $2.4 \pm 0.1$ 
  	& $41.3 \pm 0.8$ \\
  & \emph{Bin. Ent. Top $k$}
  	& $3.0 \pm 0.1$ 
  	& $54.5 \pm 5.6$ 
  	& $2.6 \pm 0.1$ 
  	& $48.9 \pm 1.7$ 
  	& $2.4 \pm 0.1$ 
  	& $43.1 \pm 0.6$ \\
  & \emph{E.F. Ent. Top $k$}
  	& $3.0 \pm 0.1$ 
  	& $53.2 \pm 0.9$ 
  	& $2.5 \pm 0.1$ 
  	& $50.6 \pm 2.1$ 
  	& $2.4 \pm 0.1$ 
  	& $43.3 \pm 0.3$ \\
  \cmidrule[0.15pt]{2-8}
  & \emph{Corr. Top $k$} 
  	& $\mathbf{2.7 \pm 0.1}$ 
  	& $\mathbf{71.6 \pm 1.1}$ 
  	& $\mathbf{2.4 \pm 0.1}$ 
  	& $\mathbf{69.7 \pm 1.7}$ 
  	& $\mathbf{2.3 \pm 0.1}$ 
  	& $\mathbf{66.7 \pm 1.7}$ \\
  \cmidrule[0.15pt]{1-8}
  \multirow{7}{*}{\shortstack[l]{Complex}}
  & \emph{L2X} \citep{chen2018learning}
  	& $2.6 \pm 0.1$ 
  	& $76.6 \pm 0.4$ 
  	& $2.4 \pm 0.1$ 
  	& $69.3 \pm 0.9$ 
  	& $2.4 \pm 0.1$ 
  	& $62.6 \pm 3.0$ \\
  & \emph{SoftSub} \citep{xie2019reparameterizable}
  	& $2.6 \pm 0.1$ 
  	& $79.4 \pm 1.1$ 
  	& $2.5 \pm 0.1$ 
  	& $69.5 \pm 2.0$ 
  	& $2.4 \pm 0.1$ 
  	& $60.2 \pm 7.0$ \\
  \cmidrule[0.15pt]{2-8}
  & \emph{Euclid. Top $k$}
  	& $2.6 \pm 0.1$ 
  	& $81.6 \pm 0.9$ 
  	& $2.4 \pm 0.1$ 
  	& $76.9 \pm 1.7$ 
  	& $2.3 \pm 0.1$ 
  	& $69.7 \pm 2.2$ \\
  & \emph{Cat. Ent. Top $k$}
  	& $\mathbf{2.5 \pm 0.1}$ 
  	& $83.7 \pm 0.8$ 
  	& $2.4 \pm 0.1$ 
  	& $76.5 \pm 0.9$ 
  	& $2.2 \pm 0.1$ 
  	& $65.9 \pm 1.4$ \\
  & \emph{Bin. Ent. Top $k$}
  	& $2.6 \pm 0.1$ 
  	& $81.9 \pm 0.7$ 
  	& $2.4 \pm 0.1$ 
  	& $75.7 \pm 1.2$ 
  	& $\mathbf{2.2 \pm 0.1}$ 
  	& $65.7 \pm 1.1$ \\
  & \emph{E.F. Ent. Top $k$}  	
  	& $2.6 \pm 0.1$ 
  	& $82.3 \pm 1.4$ 
  	& $2.4 \pm 0.1$ 
  	& $72.9 \pm 0.7$ 
  	& $2.3 \pm 0.1$ 
  	& $65.8 \pm 1.3$ \\
  \cmidrule[0.15pt]{2-8}
  & \emph{Corr. Top $k$} 
  	& $\mathbf{2.5 \pm 0.1}$ 
  	& $\mathbf{85.1 \pm 2.4}$ 
  	& $\mathbf{2.3 \pm 0.1}$ 
  	& $\mathbf{77.8 \pm 1.3}$  	
  	& $\mathbf{2.2 \pm 0.1}$ 
  	& $\mathbf{74.5 \pm 1.5}$ \\
  \bottomrule
  \end{tabular}}  \end{small}
\end{table} \begin{table}[t]
  \centering
  \caption{For $k$-subset selection on palate aspect, \rsmtabbrev{s} tend to outperform baseline relaxations. Test set MSE ($\times 10^{-2}$) and subset precision (\%) is shown for models selected on valid. MSE.}  \label{table:l2x_beer_palate}
  \begin{small}
  \adjustbox{max width=\textwidth}{  \begin{tabular}{@{}llcccccc@{}}
  \toprule
   & & \multicolumn{2}{c}{$k=5$} &  \multicolumn{2}{c}{$k=10$} &  \multicolumn{2}{c}{$k=15$} \\
   \cmidrule(lr){3-4} \cmidrule(lr){5-6} \cmidrule(lr){7-8}
  Model & Relaxation & MSE & Subs. Prec. &  MSE & Subs. Prec. &  MSE & Subs. Prec. \\
  \midrule
  \multirow{7}{*}{\shortstack[l]{Simple}}
  & \emph{L2X} \citep{chen2018learning}
  	& $3.5 \pm 0.1$ 
  	& $27.8 \pm 3.7$ 
  	& $3.2 \pm 0.1$ 
  	& $21.0 \pm 1.8$ 
  	& $3.0 \pm 0.1$ 
  	& $20.5 \pm 0.7$ \\
  & \emph{SoftSub} \citep{xie2019reparameterizable}
  	& $3.7 \pm 0.1$ 
  	& $23.9 \pm 1.4$ 
  	& $3.3 \pm 0.1$ 
  	& $23.5 \pm 3.7$ 
  	& $3.1 \pm 0.1$ 
  	& $20.0 \pm 1.7$ \\
  \cmidrule[0.15pt]{2-8}
    & \emph{Euclid. Top $k$}
  	& $3.5 \pm 0.1$ 
  	& $36.0 \pm 5.7$ 
  	& $3.2 \pm 0.1$ 
  	& $27.1 \pm 0.7$ 
  	& $3.0 \pm 0.1$ 
  	& $23.7 \pm 0.8$ \\
  & \emph{Cat. Ent. Top $k$}
  	& $3.6 \pm 0.1$ 
  	& $25.4 \pm 3.6$ 
  	& $3.0 \pm 0.1$ 
  	& $28.5 \pm 2.9$ 
  	& $3.0 \pm 0.1$ 
  	& $21.7 \pm 0.4$ \\
  & \emph{Bin. Ent. Top $k$}
  	& $3.6 \pm 0.1$ 
  	& $25.2 \pm 1.7$ 
  	& $3.2 \pm 0.1$ 
  	& $27.2 \pm 2.6$ 
  	& $3.0 \pm 0.1$ 
  	& $23.4 \pm 1.7$ \\
  & \emph{E.F. Ent. Top $k$}
  	& $3.6 \pm 0.1$ 
  	& $26.0 \pm 3.0$ 
  	& $3.1 \pm 0.1$ 
  	& $27.0 \pm 1.6$ 
  	& $2.9 \pm 0.1$ 
  	& $23.4 \pm 0.6$ \\
  \cmidrule[0.15pt]{2-8}
  & \emph{Corr. Top $k$} 
  	& $\mathbf{3.2 \pm 0.1}$ 
  	& $\mathbf{54.3 \pm 1.0}$ 
  	& $\mathbf{2.8 \pm 0.1}$ 
  	& $\mathbf{50.0 \pm 1.7}$ 
  	& $\mathbf{2.7 \pm 0.1}$ 
  	& $\mathbf{46.0 \pm 2.0}$ \\
  \cmidrule[0.15pt]{1-8}
  \multirow{7}{*}{\shortstack[l]{Complex}}
  & \emph{L2X} \citep{chen2018learning}
  	& $3.1 \pm 0.1$ 
  	& $47.4 \pm 1.7$ 
  	& $2.8 \pm 0.1$ 
  	& $40.8 \pm 0.6$ 
  	& $2.7 \pm 0.1$ 
  	& $34.8 \pm 0.8$ \\
  & \emph{SoftSub} \citep{xie2019reparameterizable}
  	& $3.1 \pm 0.1$ 
  	& $44.4 \pm 1.1$ 
  	& $2.8 \pm 0.1$ 
  	& $44.2 \pm 2.0$ 
  	& $2.8 \pm 0.1$ 
  	& $38.7 \pm 1.0$ \\
  \cmidrule[0.15pt]{2-8}
  & \emph{Euclid. Top $k$}
  	& $2.9 \pm 0.1$ 
  	& $56.2 \pm 0.7$ 
  	& $2.7 \pm 0.1$ 
  	& $43.9 \pm 1.7$ 
  	& $\mathbf{2.6 \pm 0.1}$ 
  	& $38.0 \pm 1.1$ \\
  & \emph{Cat. Ent. Top $k$}
  	& $2.9 \pm 0.1$ 
  	& $55.1 \pm 0.7$ 
  	& $2.7 \pm 0.1$ 
  	& $45.2 \pm 0.8$ 
  	& $\mathbf{2.6 \pm 0.1}$ 
  	& $40.2 \pm 0.9$ \\
  & \emph{Bin. Ent. Top $k$}
  	& $2.9 \pm 0.1$ 
  	& $55.6 \pm 0.8$ 
  	& $2.7 \pm 0.1$ 
  	& $47.6 \pm 1.0$ 
  	& $2.7 \pm 0.1$ 
  	& $39.1 \pm 1.0$ \\
  & \emph{E.F. Ent. Top $k$}  	
  	& $2.9 \pm 0.1$ 
  	& $56.3 \pm 0.3$ 
  	& $2.7 \pm 0.1$ 
  	& $48.1 \pm 1.3$ 
  	& $\mathbf{2.6 \pm 0.1}$ 
  	& $40.3 \pm 1.0$ \\
  \cmidrule[0.15pt]{2-8}
  & \emph{Corr. Top $k$} 
  	& $\mathbf{2.8 \pm 0.1}$ 
  	& $\mathbf{60.4 \pm 1.5}$ 
  	& $\mathbf{2.6 \pm 0.1}$ 
  	& $\mathbf{53.5 \pm 2.9}$ 
  	& $\mathbf{2.6 \pm 0.1}$ 
  	& $\mathbf{46.8 \pm 1.5}$ \\
  \bottomrule
  \end{tabular}}  \end{small}
\end{table} \begin{table}[t]
  \centering
  \caption{For $k$-subset selection on taste aspect, MSE and subset precision tend to be lower for all methods. This is because the taste rating is highly correlated with other ratings making it difficult to identify subsets with high precision. \rsmtabbrev{s} achieve small improvements. Test set MSE ($\times 10^{-2}$) and subset precision (\%) is shown for models selected on valid. MSE.} 	
  \label{table:l2x_beer_taste}
  \begin{small}
  \adjustbox{max width=\textwidth}{  \begin{tabular}{@{}llcccccc@{}}
  \toprule
   & & \multicolumn{2}{c}{$k=5$} &  \multicolumn{2}{c}{$k=10$} &  \multicolumn{2}{c}{$k=15$} \\
   \cmidrule(lr){3-4} \cmidrule(lr){5-6} \cmidrule(lr){7-8}
  Model & Relaxation & MSE & Subs. Prec. &  MSE & Subs. Prec. &  MSE & Subs. Prec. \\
  \midrule
  \multirow{7}{*}{\shortstack[l]{Simple}}
  & \emph{L2X} \citep{chen2018learning}
  	& $3.1 \pm 0.1$ 
  	& $28.5 \pm 0.6$ 
  	& $2.9 \pm 0.1$ 
  	& $24.1 \pm 1.3$ 
  	& $2.7 \pm 0.1$ 
  	& $26.8 \pm 0.8$ \\
  & \emph{SoftSub} \citep{xie2019reparameterizable}
  	& $3.1 \pm 0.1$ 
  	& $29.9 \pm 0.8$ 
  	& $2.9 \pm 0.1$ 
  	& $27.7 \pm 0.7$ 
  	& $2.7 \pm 0.1$ 
  	& $27.8 \pm 1.9$ \\
  \cmidrule[0.15pt]{2-8}
  & \emph{Euclid. Top $k$}
  	& $3.0 \pm 0.1$ 
  	& $30.2 \pm 0.4$ 
  	& $2.7 \pm 0.1$ 
  	& $28.0 \pm 0.4$ 
  	& $2.6 \pm 0.1$ 
  	& $26.5 \pm 0.5$ \\
  & \emph{Cat. Ent. Top $k$}
  	& $3.1 \pm 0.1$ 
  	& $28.5 \pm 0.6$ 
  	& $2.8 \pm 0.1$ 
  	& $28.9 \pm 0.6$ 
  	& $2.6 \pm 0.1$ 
  	& $30.5 \pm 1.6$ \\
  & \emph{Bin. Ent. Top $k$}
  	& $3.0 \pm 0.1$ 
  	& $29.2 \pm 0.4$ 
  	& $2.9 \pm 0.1$ 
  	& $24.6 \pm 1.7$ 
  	& $2.6 \pm 0.1$ 
  	& $27.9 \pm 0.9$ \\
  & \emph{E.F. Ent. Top $k$}
  	& $3.0 \pm 0.1$ 
  	& $29.7 \pm 0.3$ 
  	& $2.7 \pm 0.1$ 
  	& $29.0 \pm 1.5$ 
  	& $2.6 \pm 0.1$ 
  	& $26.5 \pm 0.5$ \\
  \cmidrule[0.15pt]{2-8}
  & \emph{Corr. Top $k$} 
  	& $\mathbf{2.8 \pm 0.1}$ 
  	& $\mathbf{31.7 \pm 0.5}$ 
  	& $\mathbf{2.5 \pm 0.1}$ 
  	& $\mathbf{37.7 \pm 1.6}$ 
  	& $\mathbf{2.4 \pm 0.1}$ 
  	& $\mathbf{37.8 \pm 0.5}$ \\
  \cmidrule[0.15pt]{1-8}
  \multirow{7}{*}{\shortstack[l]{Complex}}
  & \emph{L2X} \citep{chen2018learning}
  	& $2.5 \pm 0.1$ 
  	& $40.3 \pm 0.7$ 
  	& $2.4 \pm 0.1$ 
  	& $42.4 \pm 2.0$ 
  	& $2.4 \pm 0.1$ 
  	& $39.7 \pm 1.1$ \\
  & \emph{SoftSub} \citep{xie2019reparameterizable}
  	& $2.5 \pm 0.1$ 
  	& $43.3 \pm 0.9$ 
  	& $2.4 \pm 0.1$ 
  	& $41.3 \pm 0.5$ 
  	& $2.3 \pm 0.1$ 
  	& $40.5 \pm 0.7$ \\
  \cmidrule[0.15pt]{2-8}
  & \emph{Euclid. Top $k$}
  	& $\mathbf{2.4 \pm 0.1}$ 
  	& $43.8 \pm 0.7$ 
  	& $2.3 \pm 0.1$ 
  	& $43.1 \pm 0.6$ 
  	& $2.2 \pm 0.1$ 
  	& $42.2 \pm 0.8$ \\
  & \emph{Cat. Ent. Top $k$}
  	& $\mathbf{2.4 \pm 0.1}$ 
  	& $\mathbf{46.5 \pm 0.6}$ 
  	& $2.3 \pm 0.1$ 
  	& $44.6 \pm 0.3$ 
  	& $2.2 \pm 0.1$ 
  	& $45.5 \pm 1.1$ \\
  & \emph{Bin. Ent. Top $k$}
  	& $\mathbf{2.4 \pm 0.1}$ 
  	& $40.9 \pm 1.3$ 
  	& $2.3 \pm 0.1$ 
  	& $46.3 \pm 0.9$ 
  	& $2.2 \pm 0.1$ 
  	& $44.7 \pm 0.5$ \\
  & \emph{E.F. Ent. Top $k$}  	
  	& $\mathbf{2.4 \pm 0.1}$ 
  	& $45.3 \pm 0.6$ 
  	& $\mathbf{2.2 \pm 0.1}$ 
  	& $46.1 \pm 0.8$ 
  	& $2.2 \pm 0.1$ 
  	& $\mathbf{46.6 \pm 1.1}$ \\
  \cmidrule[0.15pt]{2-8}
  & \emph{Corr. Top $k$} 
  	& $\mathbf{2.4 \pm 0.1}$ 
  	& $45.9 \pm 1.3$ 
  	& $\mathbf{2.2 \pm 0.1}$ 
  	& $\mathbf{47.3 \pm 0.6}$ 
  	& $\mathbf{2.1 \pm 0.1}$ 
  	& $45.1 \pm 2.0$ \\
  \bottomrule
  \end{tabular}}  \end{small}
\end{table} \subsection{REINFORCE and NVIL on Graph Layout}
We experimented with 3 variants of REINFORCE estimators, each with a different baseline. The \emph{EMA} baseline is an exponential moving average of the ELBO. The \emph{Batch} baseline is the mean ELBO of the current batch. Finally, the \emph{Multi-sample} baseline is the mean ELBO over $k$ multiple samples, which is a local baseline for each sample (See section 3.1 of \citep{kool2019buy}). For NVIL, the input-dependent baseline was a one hidden-layer MLP with ReLU activations, attached to the GNN encoder, just before the final fully connected layer. We did not do variance normalization. We used weight decay on the encoder parameters, including the input-dependent baseline parameters. We tuned weight decay and the exponential moving average constant, in addition to the learning rate. For \emph{Multi-sample} REINFORCE, we additionally tuned $k=\{2, 4, 8\}$, and following \citep{kool2019buy}, we divided the batch size by $k$ in order to keep the number of total samples constant. 

We used $U$ as the ``action’’ for all edge distributions, and therefore, computed the log probability over $U$. We also computed the KL divergence with respect to $U$ as in the rest of the graph layout experiments (See App. \ref{supp:subsubsec:graph_layout_model}).This was because computing the probability of $X$ is not computationally efficient for \emph{Top $|V|-1$} and \emph{Spanning Tree}. In particular, the marginal of $X$ in these cases is not in the exponential family. We emphasize that using $U$ as the ``action'' for REINFORCE is atypical.

We found that both NVIL and REINFORCE with \emph{Indep. Directed Edges} and \emph{Top $|V|-1$} perform similarly to their SST counterparts, struggling to learn the underlying structure. This is also the case for REINFORCE with \emph{Spanning Tree}. On the other hand, NVIL with \emph{Spanning Tree}, is able to learn some structure, although worse and higher variance than its SST counterpart.

\begin{table}[t!]
\caption{NVIL and REINFORCE struggle to learn the underlying structure wherever their \rsmtabbrev{} counterparts struggle. NVIL with \emph{Spanning Tree} is able to learn some structure, but it is still worse and higher variance than its \rsmtabbrev{} counterpart. This is for $T=10$.}
\label{table:layout_nvil}
 \begin{small}
 \centering
 \adjustbox{max width=\textwidth}{ \begin{tabular}{@{}lccccccccc@{}}
 \toprule
  & \multicolumn{3}{c}{REINFORCE (\emph{EMA})} & \multicolumn{3}{c}{NVIL} \\ 
  \cmidrule(l){2-4} 
  \cmidrule(l){5-7}     
  Edge Distribution & ELBO & Edge Prec. & Edge Rec. & ELBO & Edge Prec. & Edge Rec. \\ 
  \midrule
 \emph{Indep. Directed Edges} 
 & $-1730 \pm 60$
 & $41 \pm 4$
 & $92 \pm 7$
 & $-1550 \pm 20$
 & $44 \pm 1$
 & $94 \pm 1$ \\
 \emph{Top $|V|-1$}
 & $-2170 \pm 10$
 & $42 \pm 1$
 & $42 \pm 1$
 & $-2110 \pm 10$
 & $42 \pm 2$
 & $42 \pm 2$\\
 \emph{Spanning Tree}
 & $-2250 \pm 20$
 & $40 \pm 7$
 & $40 \pm 7$
 & $-1570 \pm 300$
I  & $83 \pm 20$
 & $83 \pm 20$\\ 
 \bottomrule
 \end{tabular}} \end{small}
 
 \begin{small}
 \centering
 \adjustbox{max width=\textwidth}{ \begin{tabular}{@{}lccccccccc@{}}
 \toprule
  & \multicolumn{3}{c}{REINFORCE (\emph{Batch})} & \multicolumn{3}{c}{REINFORCE (\emph{Multi-sample})} \\ 
  \cmidrule(l){2-4} 
  \cmidrule(l){5-7}     
  Edge Distribution & ELBO & Edge Prec. & Edge Rec. & ELBO & Edge Prec. & Edge Rec. \\ 
  \midrule
 \emph{Indep. Directed Edges} 
 & $-1780 \pm 20$
 & $39 \pm 3$
 & $90 \pm 6$
 & $-1710 \pm 30$
 & $38 \pm 3$
 & $88 \pm 6$\\
 \emph{Top $|V|-1$}
 & $-2180 \pm 0$
 & $39 \pm 1$
 & $39 \pm 1$
 & $-2150 \pm 10$
 & $40 \pm 0$
 & $40 \pm 0$\\
 \emph{Spanning Tree}
 & $-2260 \pm 0$
 & $41 \pm 1$
 & $41 \pm 1$
 & $-2230 \pm 20$
 & $42 \pm 1$
 & $42 \pm 1$\\ 
 \bottomrule
 \end{tabular}} \end{small}
\end{table}
  \section{Experimental Details}
\label{supp:sec:exp_details}
\subsection{Implementing Relaxed Gradient Estimators}
For implementing the relaxed gradient estimator given in \eqref{eq:solution}, several options are available. In general, the forward computation of $X_{\temp}$ may be unrolled, such that the estimator can be computed with the aid of modern software packages for automatic differentiation \citep{abadi2016tensorflow, paszke2017automatic, jax2018github}. However, for some specific choices of $f$ and $\discreteset$, it may be more efficient to compute the estimator exactly via a custom backward pass, e.g. \citep{amos2019limited, martins2017learning}. Yet another alternative is to use local finite difference approximations as pointed out by \citep{domke2010impdiff}. In this case, an approximation for $d \loss(X_{\temp})/ d U$ is given by
\begin{equation}
    \frac{d \loss(X_{\temp})}{d U} \approx \frac{X_{\temp}(U + \epsilon \partial \loss(X_{\temp}) / \partial X_{\temp}) - X_{\temp}(U - \epsilon \partial \loss(X_{\temp}) / \partial X_{\temp})}{2\epsilon}
\end{equation}
with equality in the limit as $\epsilon \to 0$. This approximation is valid, because the Jacobian of $X_{\temp}$ is symmetric \citep[][Cor. 2.9]{rockafellar1999second}. It is derived from the vector chain rule and the definition of the derivative of $X_{\temp}$ in the direction $\partial \loss(X_{\temp}) / \partial X_{\temp}$. This method only requires two additional calls to a solver for \eqref{eq:rsmt} and does not require additional evaluations of $\loss$. We found this method helpful for implementing \emph{E.F. Ent. Top $k$} and \emph{Corr. Top $k$}.

\subsection{Numerical Stability}
Our \rsmtabbrev{}s for undirected and rooted direct spanning trees (\emph{Spanning Tree} and \emph{Arborescence}) require the inversion of a matrix. We found matrix inversion prone to suffer from numerical instabilities when the maximum and minimum values in $\theta$ grew too large apart. As a resolution, we found it effective to cap the maximal range in $\theta$ to $15$ during training. Specifically, if $\theta_{\text{max}} = \max(\theta)$, we clipped, i.e., $\theta_i = \max(\theta_i, \theta_{\text{max}} - 15)$. In addition, after clipping we normalized, i.e., $\theta = \theta - \theta_{\text{max}}$. This leaves the computation unaffected but improves stability. In addition, for \emph{Spanning Tree} we chose the index $k$ (c.f., Section \ref{supp:sec:fieldguide}) to be the row in which $\theta_{\text{max}}$ occurs. We did not clip when evaluating the models.

\subsection{Estimating Standard Errors by Bootstrapping Model Selection}
For all our experiments, we report standard errors over the model selection process from bootstrapping. In all our experiments we randomly searched hyperparameters over $N=20$ (NRI, ListOps) or $N=25$ (L2X) independent runs and selected the best model over these runs based on the task objective on the validation set. In all tables, we report test set metrics for the best model thus selected. We obtained standard errors by bootstrapping this procedure. Specifically, we randomly sampled with replacement $N$ times from the $N$ runs and selected models on the sampled runs. We repeated this procedure for $M = 10^5$ times to compute standard deviations for all test set metrics over the $M$ trials.

\subsection{Computing the KL divergence}
\label{supp:subsec:kl}
There are at least 3 possible KL terms for a relaxed ELBO: the KL from a prior over $\mathcal{X}$ to the distribution of $X$,  the KL from a prior over $\hull(\mathcal{X})$ to the distribution of $X_t$, or the KL from a prior over $\mathbb{R}^n$ to the distribution of $U$, see Section C.3 of \citep{maddison2016concrete} for a discussion of this topic. In our case, since we do not know of an explicit tractable density of $X_t$ or $X$, we compute the KL with respect to $U$. The KL divergence with respect to $U$ is an \emph{upper-bound} to the KL divergence with respect to $X_t$ due to a data processing inequality. Therefore, the ELBO that we are optimizing is a \emph{lower-bound} to the relaxed variational objective. Whether or not this is a good choice is an empirical question. Note, that \emph{when optimizing the relaxed objective}, using a KL divergence with respect to $X$ does not result in a lower-bound to the relaxed variational objective, as it is not necessarily an ELBO for the continuous relaxed model (see again Section C.3 of \citep{maddison2016concrete}).\\

\subsection{Neural Relational Inference (NRI) for Graph Layout}
\subsubsection{Data}
Our dataset consisted of latent prior spanning trees over 10 vertices. Latent spanning trees were sampled by applying  Kruskal’s algorithm \citep{kruskal1956shortest} to $U \sim \Gumbel(0)$ for a fully-connected graph. Note that this does not result in a uniform distribution over spanning trees. Initial vertex locations were sampled from $\Normal(0, I)$ in $\R^2$. Given initial locations and the latent tree, dynamical observations were obtained by applying a force-directed algorithm for graph layout \citep{fruchterman1991graph} for $T\in\{10, 20\}$ iterations. We then discarded the initial vertex positions, because the first iteration of the layout algorithm typically results in large relocations. This renders the initial vertex positions an outlier which is hard to model. Hence, the final dataset used for training consisted of 10 respectively 20 location observations in $\R^2$ for each of the 10 vertices. By this procedure, we generated a training set of size 50,000 and validation and test sets of size 10,000.

\subsubsection{Model}
\label{supp:subsubsec:graph_layout_model}
The NRI model consists of encoder and decoder graph neural networks. Our encoder and decoder architectures were identical to the MLP encoder and  MLP decoder architectures, respectively, in \citep{kipf2018neural}.

\paragraph{Encoder}
The encoder GNN passes messages over the fully connected directed graph with $n=10$ nodes. We took the final edge representation of the GNN to use as $\theta$. The final edge representation was in $\R^{90 \times m}$, where $m=2$ for \emph{Indep. Directed Edges} and $m=1$ for \emph{E.F. Ent. Top $|V|-1$} and \emph{Spanning Tree}, both over undirected edges (90 because we considered all directed edges excluding self-connections). We had $m=2$ for \emph{Indep. Directed Edges}, because we followed \citep{kipf2018neural} and applied the Gumbel-Max trick independently to each edge. This is equivalent to using $U \sim \Logistic(\theta)$, where $\theta \in \R^{90}$. Both \emph{E.F. Ent. Top $|V|-1$} and \emph{Spanning Tree} require undirected graphs, therefore, we ``symmetrized’’ $\theta$ such that $\theta_{ij} = \theta_{ji}$ by taking the average of the edge representations for both directions. Therefore, in this case, $\theta \in \R^{45}$.

\paragraph{Decoder}
Given previous timestep data, the decoder GNN passes messages over the sampled graph adjacency matrix $X$ and predicts future node positions. As in \citep{kipf2018neural}, we used teacher-forcing every 10 timesteps. $X \in \R^{n \times n}$ in this case was a directed adjacency matrix over the graph $G = (V, E)$ where $V$ were the nodes. $X_{ij} = 1$ is interpreted as there being an edge from $i \to j$ and $0$ for no edge. For the \smtabbrev{s} over undirected edges (\emph{E.F. Ent. Top $|V|-1$} and \emph{Spanning Tree})  $X$ was the symmetric, directed adjacency matrix with edges in both directions for each undirected edge. The decoder passed messages between both connected and not-connected nodes. When considering a message from node $i \to j$, it used one network for the edges with $X_{ij} = 1$ and another network for the edges with $X_{ij} = 0$, such that we could differentiate the two edge ``types''. For the \rsmtabbrev{} relaxation, both messages were passed, weighted by $(X_{t})_{ij}$ and $1-(X_{t})_{ij}$, respectively. Because of the parameterization of our model, during evaluation, it is ambiguous whether the sampled hard graph is in the correct representation (adjacency matrix where 1 is the existence of an edge, and 0 is the non-existence of an edge). Therefore, when measuring precision and recall for structure discovery, we selected whichever graph (the sampled graph versus the graph with adjacency matrix of one minus that of the sampled graph) that yielded the highest precision, and reported precision and recall measurements for that graph.

\paragraph{Objective}
Our ELBO objective consisted of the reconstruction error and KL divergence. The reconstruction error was the Gaussian log likelihood of the predicted node positions generated from the decoder given ground truth node positions. As mentioned in \ref{supp:subsec:kl}, we computed the KL divergence with respect to $U$ instead of the sampled graph for all methods, because computing the probability of a \emph{Spanning Tree}, or \emph{Top $k$} sample is not computationally efficient. We chose our prior to be $\Gumbel(0)$. The KL divergence between a Gumbel distribution with location $\theta$ and a Gumbel distribution with location 0, is $\theta + \exp(-\theta) - 1$.

\subsubsection{Training}
All graph layout experiments were run with batch size 128 for 50000 steps. We evaluated the model on the validation set every 500 training steps, and saved the model that achieved the best average validation ELBO. We used the Adam optimizer with a constant learning rate, and $\beta_1=0.9, \beta_2=0.999, \epsilon=10^{-8}$. We tuned hypermarameters using random uniform search over a hypercube-shaped search space with 20 trials. We tuned the constant learning rate, and temperature $t$ for all methods. For \emph{E.F. Ent. Top $k$}, we additionally tuned $\epsilon$, which is used when computing the gradients for the backward-pass using finite-differences. The ranges for hyperparameter values were chosen such that optimal hyperparameter values (corresponding to the best validation ELBO) were not close to the boundaries of the search space.
 \subsection{Unsupervised Parsing on ListOps}
\subsubsection{Data}
We considered a simplified variant of the ListOps dataset \citep{nangia2018listops}. Specifically, we used the same data generation process as \citep{nangia2018listops} but excluded the \texttt{summod} operator and used rejection sampling to ensure that the lengths of all sequences in the dataset ranged only from 10 to 50 and that our dataset contained the same number of sequences of depths $d \in \{1, 2, 3, 4, 5\}$. Depth was measured with respect to the ground truth parse tree. For each sequence, the ground truth parse tree was defined by directed edges from all operators to their respective operands. We generated 100,000 samples for the training set (20,000 for each depth), and 10,000 for the validation and test set (2,000 for each depth).

\subsubsection{Model}
We used an embedding dimension of 60, and all neural networks had 60 hidden units. 

\paragraph{LSTM}
We used a single-layered LSTM going from left to right on the input embedding matrix. The LSTM had hidden size 60 and includes dropout with probability 0.1. The output of the LSTM was flattened and fed into a single linear layer to bring the dimension to 60. The output of the linear layer was fed into an MLP with one hidden layer and ReLU activations.

\paragraph{GNN on latent (di)graph}
Our models had two main parts: an LSTM encoder that produced a graph adjacency matrix sample ($X$ or $X_t$), and a GNN that passed messages over the sampled graph. 

The LSTM encoder consisted of two LSTMs-- one representing the ``head’’ tokens, and the other for ``modifier’’ tokens. Both LSTMs were single-layered, left-to-right, with hidden size 60, and include dropout with probability 0.1. Each LSTM outputted a single real valued vector for each token $i$ of $n$ tokens with dimension 60. To obtain $\theta \in \R^{n \times n}$, we defined $\theta_{ij} = v_i^{\text{head}T}v_j^{\text{mod}}$, where $v_i^{\text{head}}$ is the vector outputted by the head LSTM for word $i$ and $v_i^{\text{mod}}$ is the vector outputted by the modifier LSTM for word $j$. As in the graph layout experiments, we symmetrized $\theta$ for the SSTs that require undirected edges (\emph{Indep. Undirected Edges}, and \emph{Spanning Tree}). For exponential $U \sim \exponential(\theta)$, $\theta$ was parameterized as the softplus function of the $\R^{n \times n}$ matrix output of the encoder. We used the Torch-struct library \citep{rush2020torch} to obtain soft samples for \emph{arborescence}.

$X \in \R^{n \times n}$ in this case was a directed adjacency matrix over the graph $G = (V, E)$ where $V$ were the tokens. $X_{ij} = 1$ is interpreted as there being an edge from $i \to j$ and $0$ for no edge. For the \smtabbrev{s} over undirected edges (\emph{Indep. Undirected Edges} and \emph{Spanning Tree}) $X$ was the symmetric, directed adjacency matrix with edges in both directions for each undirected edge. For \emph{Arborescence}, we assumed the first token is the root node of the arborescence. 

Given $X$, the GNN ran 5 message passing steps over the adjacency matrix, with the initial node embeddings being the input embedding. The GNN architecture was identical to the GNN decoder in the graph layout experiments, except we did not pass messages on edges with $X_{ij} = 0$ and we did not include the last MLP after every messaging step. For the SST, we simply weighted each message from $i \to j$ by $(X_t)_{ij}$. We used dropout with probability 0.1 in the MLPs. We used a recurrent connection after every message passing step. The LSTM encoder and the GNN each had their own embedding lookup table for the input. We fed the node embedding of the first token to an MLP with one hidden layer and ReLU activations. 

\subsubsection{Training}
All ListOps experiments were run with batch size 100 for 50 epochs. We evaluated the model on the validation set every 800 training steps, and saved the model that achieved the best average validation task accuracy. We used the AdamW optimizer with a constant learning rate, and $\beta_1=0.9, \beta_2=0.999, \epsilon=10^{-8}$. We tuned hypermarameters using random uniform search over a hypercube-shaped search space with 20 trials. We tuned the constant learning rate, temperature $t$, and weight decay for all methods. The ranges for hyperparameter values were chosen such that optimal hyperparameter values (corresponding to the best validation accuracy) were not close to the boundaries of the search space. \subsection{Learning To Explain (L2X) Aspect Ratings}
\label{supp:exp:l2x}
\paragraph{Data.}
We used the BeerAdvocate dataset \citep{mcauley2012learning}, which contains reviews comprised of free-text feedback and ratings for multiple aspects, including appearance, aroma, palate, and taste. For each aspect, we used the same de-correlated subsets of the original dataset as \citep{lei2016rationalizing}. The training set for the aspect appearance contained $80$k reviews and for all other aspects $70$k reviews. Unfortunately, \citep{lei2016rationalizing} do not provide separate validation and test sets. Therefore for each aspect, we split their heldout set into two evenly sized validation and test sets containing $5$k reviews each. We used pre-trained word embeddings of dimension $200$ from \citep{lei2016rationalizing} to initialize all models. Each review was padded/ cut to 350 words. For all aspects, subset precision was measured on the same subset of 993 annotated reviews from \citep{mcauley2012learning}. The aspect ratings were normalized to the unit interval $[0,1]$ and MSE is reported on the normalized scale.

\paragraph{Model.}
Our model used convolutional neural networks to parameterize both the subset distribution and to make a prediction from the masked embeddings. For parameterizing the masks, we considered a simple and (a more) complex architecture. The simple architecture consisted of a Dropout layer (with $p=0.1$) and a convolutional layer (with one filter and a kernel size of one) to parameterize $\theta_i \in \R$ for each word $i$, producing a the vector $\theta \in \R^n$. For \emph{Corr. Top $k$}, $\theta \in \R^{2n-1}$. The first $n$ dimensions correspond to each word $i$ and are parameterized as above. For dimensions $i \in \{n+1, \ldots, 2n-1\}$, $\theta_i$ represents a coupling between words $i$ and $i+1$, so we denote this $\theta_{i,i+1}$. It was parameterized as the sum of three terms, $\theta_{i,i+1} = \phi_i + \phi_i' + \phi_{i, i+1}$: $\phi_i \in \R$ computed using a seperate convolutional layer of the same kind as described above, $\phi_i' \in \R$ computed using yet another convolutional layer of the same kind as described above, and $\phi_{i,i+1} \in \R$ obtained from a third convolutional layer with one filter and a kernel size of two. In total, we used four separate convolutional layers to parameterize the simple encoder. For the complex architecture, we used two additional convolutional layers, each with $100$ filters, kernels of size three, ReLU activations to compute the initial word embeddings. This was padded to maintain the length of a review. 

$X$ was a $k$-hot binary vector in $n$-dimensions with each dimension corresponding to a word. For \emph{Corr. Top $K$}, we ignored dimensions $i \in \{n+1, \ldots, 2n-1\}$, which correspond to the pairwise indicators. Predictions were made from the masked embeddings, using $X_i$ to mask the embedding of word $i$. Our model applied a soft (at training) or hard (at evaluation) subset mask to the word embeddings of a given review. Our model then used two convolutional layers over these masked embeddings, each with $100$ filters, kernels of size three and ReLU activations. The resulting output was max-pooled over all feature vectors. Our model then made predictions using a Dropout layer (with $p=0.1$), a fully connected layer (with output dimension 100, ReLU activation) and a fully connected layer (with output dimension one, sigmoid activation) to predict the rating of a given aspect. 

\paragraph{Training.}
We trained all models for ten epochs at minibatches of size 100. We used the Adam optimizer \citep{kingma2014adam} and a linear learning rate schedule. Hyperparameters included the initial learning rate, its final decay factor, the number of epochs over which to decay the learning rate, weight decay and the temperature of the relaxed gradient estimator. Hyperparameters were optimized for each model using random search over 25 independent runs. The learning rate and its associated hyperparameters were sampled from $\{1, 3, 5, 10, 30, 50, 100\} \times 10^{-4}$, $\{1, 10, 100, 1000\} \times 10^{-4}$ and $\{5, 6, \ldots, 10\}$ respectively. Weight decay was sampled from $\{0, 1, 10, 100\} \times 10^{-6}$ and the temperature was sampled from $[0.1, 2]$. For a given run, models were evaluated on the validation set at the end of each epoch and the best validated model was was retained. For \emph{E.F. Ent. Top $k$} and \emph{Corr. Top $k$}, we  trained these methods with $\epsilon \in \{1, 10, 100, 1000\} \times 10^{-3}$ and selected the best $\epsilon$ on the validation set. We believe that it may be possible to improve on the results we report with an efficient exact implementation of the backward pass for these two methods. We found the overhead created by automatic differentiation software to differentiate through the unrolled dynamic program was prohibitively large in this experiment. 

   \end{document}